\documentclass[]{article}
\usepackage[utf8]{inputenc} 
\usepackage[T1]{fontenc}   
\usepackage[margin=1in]{geometry}
\usepackage{hyperref}       
\hypersetup{
    pdftitle = {Deep Learning of Compositional Targets with Hierarchical Spectral Methods},
    pdfauthor = {Hugo Tabanelli, Yatin Dandi, Luca Pesce, Florent Krzakala},
    hypertexnames = false,
    colorlinks = true,
    citecolor = blue,
    linkcolor = red
}

\usepackage{url}            
\usepackage{booktabs}       
\usepackage{amsfonts}     
\usepackage{nicefrac}       
\usepackage{microtype}      
\usepackage{xcolor}         
\usepackage{color}
\usepackage{float}
\usepackage[style=base]{caption}
\usepackage{subcaption}
\usepackage{tikz}
\usetikzlibrary{arrows,positioning,shapes.geometric}
\usepackage{tikz-cd}
\usepackage{pgfplots}
\usepackage{abstract}
\usepgfplotslibrary{groupplots}
\pgfplotsset{compat=1.18}

\usepackage{amsmath,amssymb,amsthm,mathtools,bbm}
\usepackage{physics}
\usepackage{bbold}
\usepackage{dsfont}
\usepackage{mathrsfs}
\usepackage[capitalize,noabbrev]{cleveref}
\usepackage{xfrac}
\usepackage{graphicx}
\usepackage{authblk}
\usepackage[authoryear,round]{natbib}

\usepackage{times}
\usepackage{comment}

\usepackage{enumitem}
\setlist[enumerate]{leftmargin=1.5em}
\setlist[itemize]{leftmargin=1.5em}
\usepackage{autonum}

\usepackage[normalem]{ulem}  
\usepackage{changepage}    
\usepackage{bm}  

\numberwithin{equation}{section}
\usepackage{thmtools}
\usepackage{mdframed}

\usepackage{xfrac}
\usepackage{algorithm}
\usepackage{algorithmic}
\usepackage{adjustbox}
\usepackage{caption}
\usepackage{float}

\renewcommand{\algorithmicrequire}{\textbf{Input:}}
\renewcommand{\algorithmicinput}{\textbf{Input:}}
\newcommand{\RETURN}{\STATE \textbf{Return }}

\BeforeBeginEnvironment{algorithm}{\LetLtxMacro\arxivSavedLabel\label}
\AfterEndEnvironment{algorithm}{\GlobalLetLtxMacro\label\arxivSavedLabel}

\declaretheorem[thmbox=M,name=Theorem,numberwithin=section]{theorem}

\declaretheorem[thmbox=S,name=Lemma,numberwithin=section]{lemma}

\theoremstyle{definition}

\newtheorem{remark}{Remark}[section]
\newtheorem{assumption}{Assumption}[section]

\makeatletter
\renewenvironment{proof}[1][\proofname]{\par
  \pushQED{\qed}%
  \normalfont \topsep6\p@\@plus6\p@\relax
  \trivlist
  \item[\hskip\labelsep\textbf{#1.}]\ignorespaces
}{%
  \popQED\endtrivlist\@endpefalse
}
\makeatother

\newcommand{\R}{\mathbb{R}}
\newcommand{\bx}{{\bf{x}}}
\newcommand{\bh}{{\bf{h}}}

\newcommand{\bv}{{\bf{v}}}

\providecommand{\norm}[1]{\left\lVert#1\right\rVert}

  \providecommand{\R}{\mathbb{R}} %

\begin{document}

\title{Deep Learning of Compositional Targets\\
with Hierarchical Spectral Methods}

\author[1]{Hugo Tabanelli}
\author[1,2]{Yatin Dandi}
\author[1]{Luca Pesce}
\author[1]{Florent Krzakala}

\affil[1]{Information, Learning and Physics Laboratory, \'Ecole polytechnique f\'ed\'erale de Lausanne (EPFL)}
\affil[2]{Statistical Physics of Computation Laboratory, \'Ecole polytechnique f\'ed\'erale de Lausanne (EPFL)}

\maketitle

\begin{abstract}
Why depth yields a computational advantage over shallow methods remains a central open question in learning theory. We study this question in a controlled high-dimensional Gaussian setting, focusing on compositional target functions.
We analyze their learnability using an explicit three-layer fitting model trained via layer-wise spectral estimators. The target combines polynomial latent features with a scalar readout, and its compositional structure allows learning to proceed in stages. Recovering an intermediate representation reduces the problem to lower-degree spectral estimation tasks, as in multi-index models. Our analysis establishes Gaussian comparison bounds for projected moments when $0<\varepsilon<1/3$, and motivates separations in sample complexity between two- and three-layer learning strategies, conditional on a conjectured spectral recovery step.
\end{abstract}

\section{Introduction}
\label{sec:intro}

Deep neural networks consistently outperform shallow models across a wide range of learning tasks, yet providing a rigorous and quantitative explanation for the computational advantages conferred by depth remains a central open problem in machine learning theory \citep{sejnowski2020unreasonable,zhang2021understanding}.  
While classical approximation theory establishes that deep architectures can represent certain functions more efficiently than shallow ones \citep{pmlr-v49-telgarsky16,mhaskar2017and,poggio2017and}, approximation alone does not explain why such functions should be \emph{learnable} from data using feasible sample sizes. A fundamental theoretical challenge can thus be phrased as:
\begin{center}
\textit{
\textbf{Q1:} Can one  quantify the computational advantages of depth for learning structured targets in high dimensions?
}
\end{center}

Several lines of work have provided partial answers in controlled settings.  
Deep \emph{linear} networks admit an exact characterization of training dynamics and implicit bias \citep{saxe2014exact,ji2019gradient,arora2019convergence,lee2019wide,ghorbani2021linearized}, but their limited expressivity prevents the emergence of nonlinear hierarchical features.
Another successful direction studies multi-index models, where the target depends on the input only through a fixed low-dimensional projection
\citep{BenArous2021,ba2020generalization,ghorbani2020neural,bietti2022learning,abbe2023sgd,troiani2024fundamental}.
A large part of this literature is formulated in a teacher-student setting: labels are generated by a structured teacher function, and one asks whether a learning algorithm, or student model, can recover the relevant low-dimensional structure from samples.
This viewpoint also extends to broader neural-network teacher-student models \cite{gardner1989three,aubin2018committee,buzaglo2025uniform}.
We use a compositional teacher to isolate the role of depth.
While mathematically rich, multi-index targets are often efficiently learnable by shallow architectures \cite{arnaboldi2024repetita,lee2024neural}, limiting their usefulness for isolating a fundamental advantage of depth.
This motivates a more structural question:
\begin{center}
\textit{
\textbf{Q2:} What class of target functions can reveal a provable separation between shallow and deep learning strategies?
}
\end{center}

There has been some recent progress in this direction.
For instance, the study of \emph{Random Hierarchy Models} (RHMs), introduced by \citet{cagnetta2024towards} as a generative model for compositional data, provides a clean benchmark where empirical evidence suggests a dramatic separation between shallow and deep architectures in terms of sample complexity. An orthogonal  line of work concerns the \emph{computational advantage of depth through training dynamics} for structured, hierarchical target functions in high-dimensional Gaussian settings: \citet{dandicomputational} introduced a class of high-dimensional hierarchical target functions and showed that deep networks trained by gradient descent can learn them with drastically fewer samples than any shallow model. Closer to us, and a major inspiration to the present work, is the series of papers \citep{nichani2024provable,wang2023learning,fu2025learning} that discuss compositional functions and demonstrated the advantage of three-layer nets over two layers and kernel methods. 

These approaches suggest that learning intermediate features can decompose a compositional learning task into a sequence of simpler problems.

In this work, we follow this perspective in the simplest tractable setting.
We study target functions formed by composing polynomial feature maps with a scalar readout.

We replace gradient-based training by an explicit sequence of layer-wise \emph{spectral estimators}, extending prior work \cite{nichani2024provable,fu2025learning,dandicomputational}. This lets us study successive steps of feature recovery separately from optimization dynamics.

\section{Settings, Main Results, and Related Work}
\label{sec:model}

\begin{figure}[t]
\begin{center}
\includegraphics[width=0.7\linewidth]{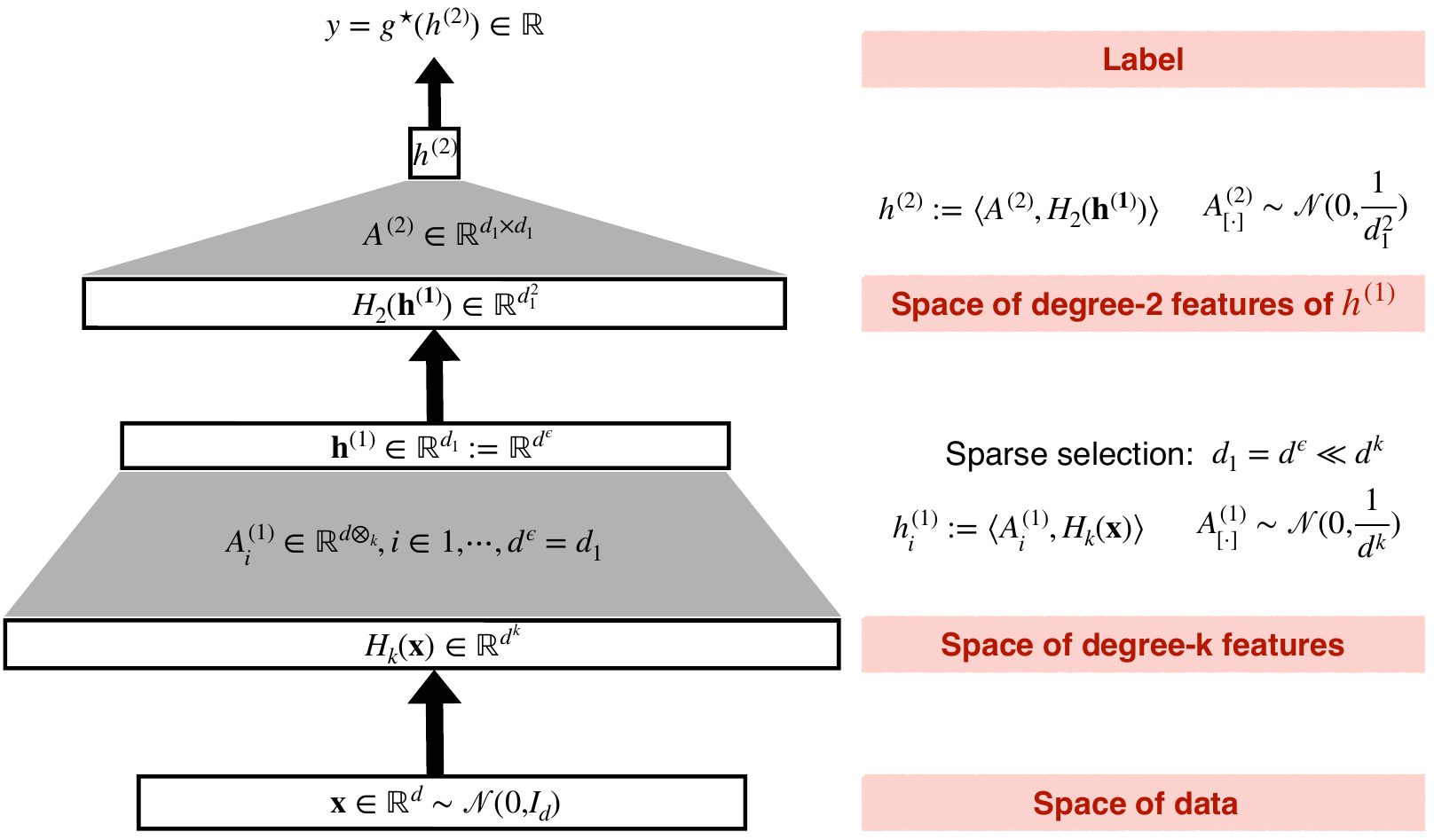}
\caption{The compositional target functions defined in Section~\ref{sec:model}.}
\end{center}
\end{figure}

\paragraph{Motivation.} In image and language tasks, the initial layers of trained networks extract features such as edge detectors or $n$-gram clusters. We model this as the selection of a \emph{sparse subspace} of low-degree functions of the inputs. Subsequent layers combine these features, motivating an iterative composition of blocks of the form:
\begin{align}
\text{High-dimensional inputs}\;\longrightarrow\;
\text{space of low-degree features}\;\longrightarrow\;\text{sparse selection}.
\end{align}

\paragraph{Data and Target Functions.} We consider an i.i.d. dataset $\{(\bx_\mu,y_\mu)\}_{\mu=1}^n$ with Gaussian inputs $\bx_\mu\sim\mathcal N(0,I_d)$ and labels $y_\mu=f^\star(\bx_\mu)$. The target applies a scalar readout to hierarchical polynomial features:
\begin{align}
\bx_\mu \in \mathbb{R}^d&\;\longrightarrow\; \bh^{(1)}_\mu \in \mathbb{R}^{d^\epsilon}  \;\longrightarrow\; 
h^{(2)}_\mu \in \mathbb{R}, 
\end{align}
defined as
\begin{align}
\label{eq:feature}
h^{(1)}_{i} &:= \langle A^{(1)}_i,\, H_{k}(\bx)\rangle\,, \quad i\!=\! 1,\dots,d^{\epsilon} = d_1\,, \\
h^{(2)} &:= \langle A^{(2)}, H_2(\bh^{(1)})\rangle , \quad d_2 =1 
\end{align}
The labels are obtained by applying a readout $g^\star$, subject to Assumption~\ref{ass:2} for the rigorous results below:
\begin{align}
    y_\mu &= f^\star(\bx_\mu) = g^
\star(h^{(2)}_\mu)
\end{align}
The tensors $\{A^{(1)}_i\in\mathbb{R}^{d \times \cdots \times d}\}_{i=1}^{d^{\epsilon}}$ are independent and symmetric, with Gaussian entries independent up to symmetry and of variance $d^{-k}$, and have an effective number of parameters $D_1 = \mathcal{B}(d,k) = \binom{d+k-1}{k}$. The mapping $\bx \rightarrow h^{(1)}_{i} := \langle A^{(1)}_i,\, H_{k}(\bx)\rangle$ can be interpreted as a sparse selection of features in the space of degree-$k$ polynomials defined in $\bx$. Independently of the first layer and the data, the matrix $A^{(2)} \in\mathbb{R}^{d_1 \times d_1}$ is symmetric with independent upper-triangular entries distributed as $\mathcal{N}(0,d_1^{-2})$. This normalization gives $\|A^{(2)}\|_F^2=1+\tilde{O}(d_1^{-1})$ and $\|A^{(2)}\|_2=\mathcal{O}(d_1^{-1/2})$ with high probability.

Here $H_k$ denotes the order-$k$ Hermite tensor
\footnote{Formally defined by the tensorial Rodrigues formula
$\sqrt{k!} H_k(\bx) = (-1)^k\, e^{\|\bx\|^2/2}\, \nabla^{\otimes k} 
\!\left( e^{-\|\bx\|^2/2} \right)
$
where $\nabla^{\otimes k}$ denotes the $k$-fold symmetric tensor of derivatives.}. 
The bracket $\langle \cdot,\cdot\rangle$ denotes the Frobenius inner product $\langle U,V\rangle := \mathrm{Tr}(U^\top V)$. This reads for order-$2$ Hermite tensors
 \begin{align}
 \langle A,\, H_2(\bh)\rangle \;=\; \frac 1{\sqrt{2}} \left (\bh^\top A \bh - \mathrm{Tr}(A)\right)
 \end{align}
For polynomial $g^\star$, the target $f^\star$ is a polynomial in $\bx$. We consider the high-dimensional limit
\begin{align}
d\to\infty,
\qquad
d_{1} \propto d^{\varepsilon},
\qquad
n \propto d^{\alpha},
\end{align}
with exponents $\epsilon,\alpha\ge0$. The models in \cite{nichani2024provable,wang2023learning} also study hierarchical feature learning; here the intermediate representation has a growing dimension $d_1=d^\epsilon$. For $k=2$, the target takes the form
\begin{align}
    h^{(1)}_{\mu,i} &\;:=\; \langle A^{(1)}_i,\, H_2(\bx_\mu)\rangle,\qquad i=1,\dots,d_1,\\
h^{(2)}_{\mu} &\;:=\; \langle A^{(2)},\, H_2(\bh^{(1)}_\mu)\rangle \qquad d_2 = 1  \\
y_\mu &= f^\star(\bx_{\mu}) = g^\star(h^{(2)}_{\mu})\,.
\label{eq:main:simple_target}
\end{align}
When $g^\star$ is the identity, $f^\star$ is quartic in $\bx_\mu$ and is obtained by composing two quadratic maps.

\paragraph{Expected performance: heuristic.}
Our goal is to efficiently learn the target function $f^\star(\bx)$ from the dataset $\mathcal{D} = \{(\bx_\mu,y_\mu)\}_{\mu=1}^n$, and to understand how depth can improve sample complexity.

We begin by recalling a fundamental limitation of \textbf{kernel methods and random feature models} \citep{rahimi2007random}.  
Such methods can only learn a polynomial approximation of degree $\kappa$ in the Hermite expansion of $f^\star$ when the number of samples scales as $n = \mathcal{O}(d^\kappa)$ \citep{mei2022generalization}.  
As a consequence, these approaches are insensitive to the compositional structure of the target and depend only on its overall polynomial degree.  
For instance, in the simple example~\eqref{eq:main:simple_target}, if we use $g^\star(x)=x$, then $f^\star$ is quartic in $\bx$, and kernel methods require $n = \mathcal{O}(d^4)$ samples.  
More generally, if the outer nonlinearity $g^\star$ is a polynomial of degree $p$, the required sample size scales as $\mathcal{O}(d^{4p})$.

In contrast, multi-layer strategies exploit the compositional structure of the target. In a related regime, \cite{fu2025learning} show that a three-layer network can learn $f^\star$ with $n = \mathcal{O} (d^4 + d^{\varepsilon p})$, already improving over shallow methods when the outer nonlinearity has low degree. While their setting differs from ours in the intermediate feature dimension and the upper-layer structure, this result points to the same general phenomenon: depth can make a globally high-degree target accessible by revealing and exploiting intermediate structure. 

In our model, parameter counting suggests $d^{k+\varepsilon}$ samples for learning the first-layer feature map. Once this representation is recovered, the second-layer quadratic form has $\mathcal O(d^{2\varepsilon})$ parameters, and fitting the readout is a one-dimensional regression problem. When $k\ge2$, this suggests the overall sample-size scaling
\[ 
    n = \mathcal{O}(d^{k+\varepsilon}), 
\] 
which, for example, becomes $n=\mathcal{O}(d^{2+\varepsilon})$ when $k=2$.

Our hierarchical spectral procedure uses degree and rank selection to estimate the first-layer feature subspace, then estimates the quadratic form in the recovered coordinates. Inspired by spectral estimators for multi-index models~\cite{damian2022neural,damian2024computational}, it leads to the scaling above conditional on the recovery conjecture in Remark~\ref{rem:tensor}. We prove bounds for projected moments and for second-layer estimation using the true first-layer features. Appendix~\ref{sec:gd-spectral} discusses the connection with gradient-based training.

\subsection*{Summary of Main Results}
\begin{itemize}[noitemsep,leftmargin=1em,wide=0pt]
\item {\bf Hierarchical spectral recovery:} We introduce a hierarchical spectral estimator for the latent representations $\bh^{(1)}$ and $h^{(2)}$, using Hermite moment matrices and PCA-type spectral methods motivated by eigenvalue separation phenomena (e.g.\ BBP transitions \cite{baik2005phase}). We prove bounds for first-layer projected moments and for second-layer estimation using the true first-layer features, for $0<\varepsilon<1/3$. Subject to the spectral recovery conjecture in Remark~\ref{rem:tensor}, these bounds motivate recovery of $\bh^{(1)}\in\mathbb R^{d^\varepsilon}$ when
\begin{equation}
n \gg d^{k+\varepsilon},
\end{equation}
and the signal subspace is small relative to the degree-$k$ Hermite feature space,
\begin{equation}
d^\varepsilon \ll d^{k}.
\end{equation}
In the idealized setting where $\bh^{(1)}$ is known, the second-layer moment matrix consistently estimates $\nu_1 A^{(2)}$ in scaled operator norm when
\begin{equation}
n \gg d^{2\varepsilon}.
\end{equation}
Fitting the final one-dimensional nonlinearity $g^\star$ then requires no additional dimension-dependent sample exponent. For $k>1$, the first-layer step dominates the predicted sample complexity, giving the threshold scaling
\begin{align}
n_{\mathrm{tot}} = \mathcal{O}(d^{k+\varepsilon})
\end{align}
Here the consistency statements suppress logarithmic factors.

\item \textbf{Numerical validation:} 
Numerical experiments support the predicted sample-size scaling beyond the proved range $0<\varepsilon<1/3$, with transitions visible at moderate dimensions.

\item{\bf  Gaussian Equivalence:}
Our analysis gives a conditional joint Gaussian approximation for the first-layer latent features and uses it to compare their label-weighted second moments with Gaussian moments (Lemmas~\ref{lem:1dclt} and~\ref{lem:latent-moment-comparison}).
\end{itemize}

\subsection{Further Related Work}\label{sec:rel_work}

\paragraph{Random feature and kernel methods.}
A large part of the theory of neural networks comes from regimes where features are effectively fixed during training, most notably kernel methods and random feature (RF) models \citep{rahimi2007random}.  
These approaches admit sharp generalization guarantees in high dimensions \citep{gerace2020generalisation,goldt_gaussian_2021,mei2022generalization,xiao2022precise,defilippis2024dimension}, but their expressive power is fundamentally limited.  
In particular, kernel and RF methods can only recover low-degree polynomial approximations of the target, with an effective degree bounded by the number of samples and features \citep{mei2022generalization}.

\paragraph{Multi-index and spectral methods.}
Despite substantial progress in understanding fixed-feature methods, a central challenge in learning theory remains a principled description of how learning algorithms adapt to low-dimensional structure.  
A canonical framework to study this phenomenon is provided by \emph{multi-index models}, where the target depends on the input only through a small number of linear projections followed by a nonlinear map.  
The information-theoretic limits of these models are well understood \citep{barbier2019optimal,aubin2018committee}, and a large body of work has characterized their algorithmic learnability, highlighting intrinsic limitations of kernel methods \citep{mei2022generalization} and the role of algorithmic thresholds and hardness exponents for neural networks \citep{BenArous2021,abbe2022merged,ba2020generalization,dandi2024twolayer,damian2024computational,arnaboldi2024repetita,lee2024neural,montanari2026phasetransitionsfeaturelearning}.  
Recent results show that suitable algorithmic variants of SGD can achieve near-optimal sample complexity for this class \citep{arnaboldi2024repetita,lee2024neural,troiani2024fundamental}.

Our spectral estimators build on this line of work, particularly \cite{damian2024computational}. Related spectral methods estimate low-dimensional structure in high-dimensional models \citep{lu2020phase,mondelli2018fundamental,maillard2022construction}, with optimal or near-optimal guarantees for multi-index models \citep{kovavcevic2025spectral,defilippis2025optimal}.
Our work extends these to a hierarchical setting, where spectral estimation is combined with rank selection and cleaning procedures.

\paragraph{Hierarchical and compositional models.}
Depth is often argued to be effective because it allows networks to exploit hierarchical or compositional structure in the data.  In addition to the works discussed in the introduction, this intuition has motivated a variety of hierarchical target and data models, including, e.g., tree-structured functions and random hierarchy models \citep{mossel2016deep,daniely2017depth,poggio2017and,allen2019can,abbe2022merged, ren2023depthseparationmultilayermeanfield,garnierbrun2025transformerslearnstructureddata,cagnetta2024towards,cagnettarandomhierarchy}.  The idea that learning proceeds by extracting structure across successive scales is closely related to coarse-graining and renormalization concepts from statistical physics \citep{wilson1971renormalization}.  
Such connections have inspired several theoretical studies of deep learning \citep{mehta2014exact,li2018neural,marchand2023multiscale,dandicomputational}.

\paragraph{Gaussian Equivalence Principles.}
\emph{Gaussian equivalence} or \emph{universality} principles characterize learning algorithms on non-Gaussian data through a Gaussian comparison model.
Beginning with the seminal work of \citet{el2008spectrum} on the spectrum of sample covariance matrices, such Gaussian universality properties have since been established in a wide range of learning settings.  
These include generalized linear estimation and random feature models \citep{gerace2020generalisation,goldt_gaussian_2021,mei_generalization_2022,dhifallah2020precise,vilucchio2025asymptotics}, empirical risk minimization and neural tangent kernel regimes \citep{montanari2022universality}, as well as mixtures of Gaussians \citep{dandi2023universality}. Of particular relevance to our analysis are recent extensions of Gaussian equivalence to \emph{quadratic} and higher-order polynomial feature models \citep{bandeira2025exact,COLTXU2025,wen2025does,xiao2022precise,hu2024asymptotics}.  
Understanding the spectral behavior of kernel and moment matrices in these polynomial regimes has become a central topic in random matrix theory \citep{lu2025equivalence}.

\section{Hierarchical Spectral Methods}\label{sec:method}
Algorithm~\ref{alg:agnostic-recovery} estimates the latent features in \eqref{eq:feature} layer by layer using label-weighted second moments and spectral thresholding.

\begin{algorithm}[!t]
\caption{Hierarchical spectral learning}
\label{alg:agnostic-recovery}
\begin{algorithmic}[1]

\INPUT Data $\{(\bx_\mu,y_\mu)\}_{\mu=1}^n$, max degree $K_{\max}$

\STATE \textbf{First layer recovery:}
\FOR{$k=1$ to $K_{\max}$} \STATE Compute $\phi^{(k)}_\mu=F[H_k(\bx_\mu)]\in\R^{D_k}$ and
    \[ 
    \widehat C^{(1)}_k =\frac{1}{n}\sum_{\mu=1}^n y_\mu H_2\!\left(\phi^{(k)}_\mu\right) \in \R^{D_k\times D_k}. 
    \] 
\ENDFOR

\STATE Select the degree $\widehat k_1$ as the smallest $k \le K_{\max}$ for which $\widehat C^{(1)}_k$ exhibits a low-rank structure. Let $\widehat d_1$ be the corresponding rank.

\STATE Select the $\widehat d_1$ eigenvectors of $\widehat C^{(1)}_{\widehat k_1}$ with largest absolute eigenvalues, denoted $\{\widehat A^{(1)}_i\}_{i=1}^{\widehat d_1}$.

\FOR{$\mu=1$ to $n$ and $i=1$ to $\widehat d_1$}
    \STATE
    \[
        \widehat h^{(1)}_{\mu,i}
        \leftarrow
        \left\langle
        \widehat A^{(1)}_i,
        F[H_{\widehat k_1}(\bx_\mu)]
        \right\rangle .
    \]
\ENDFOR

\STATE \textbf{Second layer recovery:}
\STATE 
\[
    \widehat A^{(2)}
    \leftarrow
    \widehat C_2^{(2)} = \frac{1}{n}\sum_{\mu=1}^n
    y_\mu \, H_2\!\left(\widehat \bh^{(1)}_\mu\right)
    \in \R^{\widehat d_1 \times \widehat d_1}.
\]

\FOR{$\mu=1$ to $n$}
    \STATE 
    \[
        \widehat h^{(2)}_\mu
        \leftarrow
        \left\langle
        \widehat A^{(2)},
        H_2\!\left(\widehat \bh^{(1)}_\mu\right)
        \right\rangle .
    \]
\ENDFOR

\STATE Perform kernel regression on
$\{(\widehat h^{(2)}_\mu,y_\mu)\}_{\mu=1}^n$.

\RETURN $\widehat A^{(1)}, \widehat A^{(2)}$

\end{algorithmic}
\end{algorithm}

\begin{figure}
   \centering
\includegraphics[width=0.9\linewidth]{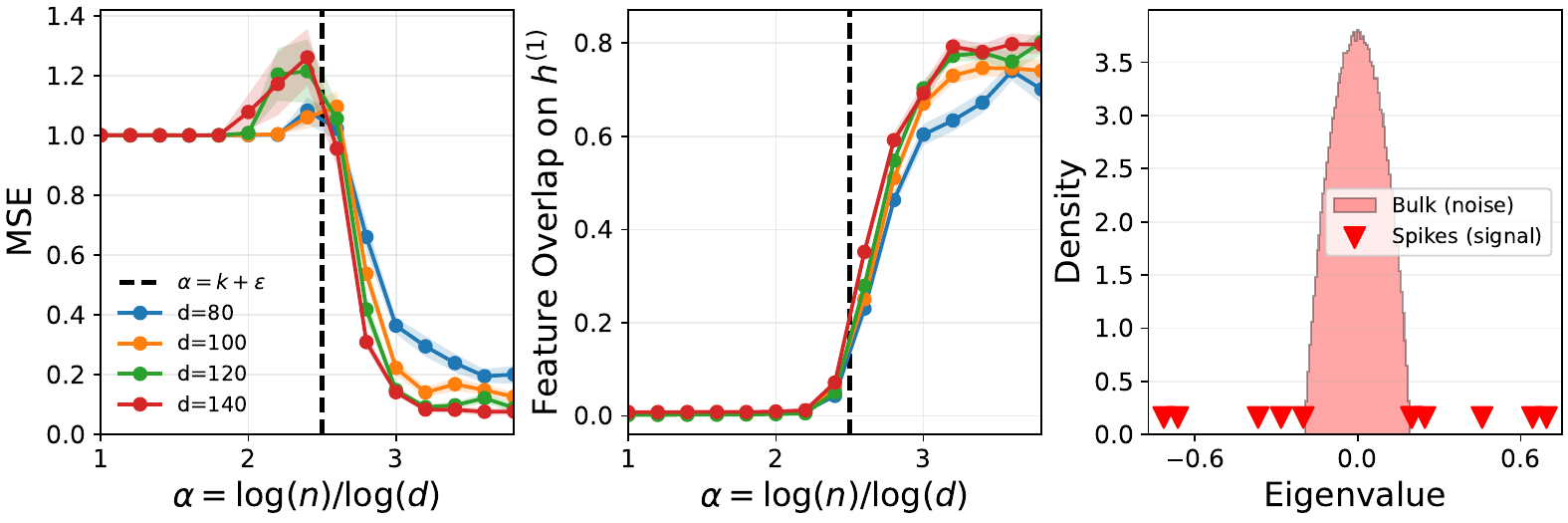}
   \caption{\textbf{Learning with hierarchical spectral methods.} Algorithm~\ref{alg:agnostic-recovery} applied to \eqref{eq:main:simple_target} with $g^\star(x)=x$ and $d_1=\sqrt d$. \textbf{Left:} Mean Squared Error (MSE) versus the sample-size exponent $\alpha=\log_d n$, for $d\in\{80,100,120,140\}$. \textbf{Center:} First-layer feature overlap, defined in Appendix~\ref{sec:app:numerics}. Both change near the predicted scale $n=d^{5/2}$; $\varepsilon=1/2$ lies outside the proved range of Theorem~\ref{thm:matrix_conc_1}. \textbf{Right:} Spectrum of $\hat C_2^{(1)}$ in \eqref{eq:main:chat1} for $d=100$ and $\alpha=3$, showing ten signal eigenvalues separated from the bulk.}
   \label{fig:main_fig_subplots}
\end{figure}
\paragraph{Recovery of the first-layer features $\bh^{(1)}$.}
Let $F$ flatten symmetric order-$k$ tensors while preserving Frobenius inner products. Then $F[H_k(\bx)]\in\mathbb R^{D_1}$, where $D_1=\mathcal B(d,k)=\mathcal O(d^k)$ is the dimension of the degree-$k$ Hermite feature space.
Our estimator constructs the empirical moment matrix
\begin{equation}
\label{eq:C-ell-k}
\widehat C^{(1)}_k
\;:=\;
\frac{1}{n}
\sum_{\mu=1}^n
y_\mu \,
H_2\!\left(F[H_k(\bx_\mu)]\right),
\end{equation}
which can be viewed as a second-order covariance operator acting on this feature space. The Gaussian comparison suggests a leading signal of rank $\mathcal{O}(d^\varepsilon)$, supported on the span of the tensors $\{A^{(1)}_i\}_{i=1}^{d^\varepsilon}$ defining the first-layer features.
Under the scaling regime
$
n \gg d^{k} d^\varepsilon
$
and
$d^{k} \gg d^\varepsilon$
we expect this signal subspace to separate from noise via a BBP-type spectral transition.

This motivates recovering the polynomial degree $k$ and the span of $\{A^{(1)}_i\}_{i=1}^{d^\epsilon}$ by eigenvalue thresholding. Our analysis controls projections along the true feature tensors; extending this control to the empirical eigenspaces is the conjecture in Remark~\ref{rem:tensor}. Section~\ref{sec:HGEM} gives the Gaussian comparison for these projections.

Consider the Gaussian comparison model with $F[H_k(\bx_\mu)]$ replaced by $\tilde{\bx}_\mu\sim\mathcal N(0,I_{D_1})$. Write $A^{(1)}\in\mathbb R^{d_1\times D_1}$ for the matrix of flattened teacher tensors, and let $\bx^\star_\mu,\bx_\mu^\perp$ be the projections of $\tilde{\bx}_\mu$ onto its row space and orthogonal complement, respectively.
Under Assumption~\ref{ass:1}, the rows of $A^{(1)}$ are asymptotically orthonormal, $A^{(1)}A^{(1)\top}=I_{d_1}+o(1)$ in operator norm in the regime considered.
In this model, the labels depend on $\tilde{\bx}_\mu$ only through $\bx^\star_\mu$.

The independence of $\bx^\star_\mu$ and $\bx_\mu^\perp$ motivates the following schematic signal-plus-noise decomposition, as for spectral estimators in single- and multi-index models \cite{lu2020phase,mondelli2018fundamental,defilippis2025optimal,kovavcevic2025spectral}:
\begin{align}
        \hat{C}_k^{(1)} &\simeq S^\star+N_k^{(1)}, \\
          S^\star &\coloneqq \nu_1 A^{(1)\top}A^{(2)}A^{(1)} \qquad\text{(candidate signal)} \label{eq:mean_C1}  \\
    N_k^{(1)} &= \frac{1}{n} \tilde{X}_{\perp} Y \tilde{X}_{\perp}^\top+o_d(1),  \label{eq:var_C1}
    \end{align}
where $\nu_1$ is the first Hermite coefficient of $g^\star$, $Y=\operatorname{diag}(y_1,\ldots,y_n)$, and $\tilde X_\perp$ has columns $\bx_\mu^\perp$. Since $d_1\ll D_1$, the rank-$d_1$ candidate signal suggests $d_1$ spikes in the $D_1\times D_1$ matrix $\hat C_k^{(1)}$.

Moreover, under the Gaussian equivalent model, the entries of $\tilde{X}_{\perp}$ are independent of the labels $y_\mu$ and hence the matrix $\frac{1}{n}\tilde{X}_{\perp} Y \tilde{X}_{\perp}^\top$ corresponds to an isotropic bulk.

The next theorem controls the projected moment $A^{(1)}\hat C_k^{(1)}A^{(1)\top}$, whose target is $\nu_1 A^{(2)}$.

Throughout, our results rely on the following assumptions:
\begin{assumption}\label{ass:1}
    For $i=1,\cdots, d_1$, $A^{(1)}_i$ are independent symmetric tensors with entries independent up to symmetry and distributed as $\mathcal{N}(0,d^{-k})$. Independently, $A^{(2)}$ has independent upper-triangular entries $\sim\mathcal{N}(0,d_1^{-2})$, reflected across the diagonal. All teacher parameters are independent of the data.
\end{assumption}
\begin{assumption}\label{ass:2}
   $g^\star:\mathbb{R} \rightarrow \mathbb{R}$ is uniformly Lipschitz and satisfies $\mathbb{E}_{z \sim \mathcal{N}(0,1)}[g^\star(z)z] \neq 0$.
\end{assumption}

\begin{theorem}\label{thm:matrix_conc_1}
Let $\hat{C}_k^{(1)}$ be as defined in Eq. \eqref{eq:C-ell-k}, with fixed $k\geq1$, $d_1\asymp d^\varepsilon$ for $0<\varepsilon<1/3$, and polynomially growing $n\gtrsim d^k$. Under Assumptions~\ref{ass:1}--\ref{ass:2}, whp over the teacher parameters and samples as $d\rightarrow\infty$:
\begin{equation}
\begin{split}
    &  \sqrt{d_1} \norm{A^{(1)}\hat{C}_k^{(1)} A^{(1)\top} -\nu_1 A^{(2)}}_2 \\&= \tilde{O}\left(\sqrt{\frac{d^{k}d_1}{n}}\right)+ \tilde{O}\left(\frac{d_1^{3/2}}{\sqrt{d}}\right) + \tilde{O}\left(\sqrt {\frac{1}{d_1}}\right),
\end{split}
\end{equation}
where $\tilde{O}$ includes polylogarithmic factors. The $\sqrt{d_1}$ scaling accounts for the normalization $\norm{A^{(2)}}_2 = \mathcal{O}(\frac{1}{\sqrt{d_1}})$. Thus the projected error vanishes when $n/(d^kd_1)$ dominates the logarithmic factors; the Gaussian-comparison term requires $d_1^3/d\to0$ up to those factors.
\end{theorem}

The proof of the theorem is provided in Appendix \ref{app:derivation}.

\begin{remark}[\textbf{Conjectured Extension}]\label{rem:tensor}
The theorem controls projections along $A^{(1)}$. We conjecture that the comparison extends to generic directions with vanishing non-trivial contractions, and that the $d_1$ eigenvectors with largest absolute eigenvalues recover the row space of $A^{(1)}$. Appendix~\ref{app:tens_nont} describes these contraction conditions; Lemma~\ref{lem:contract-scaling} verifies them for the random teacher tensors.
\end{remark}

\paragraph{Estimation of the second-layer feature $h^{(2)}$.}
Given estimates $\{\widehat A_i^{(1)}\}_{i=1}^{\widehat d_1}$, form the latent features
\begin{equation}
\label{eq:feature-recovery}
\widehat h^{(1)}_{\mu,i}
\;=\;
\big\langle
\widehat A^{(1)}_i,\,
H_{k}(\bx_\mu)
\big\rangle,
\qquad i=1,\dots,\widehat d_1 ,
\end{equation}
These form the reconstructed representation $\widehat\bh^{(1)}_\mu\in\mathbb R^{\widehat d_1}$.
The algorithm forms the empirical moment matrix
\begin{equation}
\label{eq:main:second_stage_moment}
\widehat C^{(2)}_2
\;:=\;
\frac{1}{n}
\sum_{\mu=1}^n
y_\mu \;
H_2\!\left(\widehat \bh^{(1)}_\mu\right),
\end{equation}
which acts on the reconstructed latent coordinates. We analyze the idealized estimator that uses the true first-layer features. Lemmas~\ref{lem:latent-moment-comparison} and~\ref{lem:comp_hermite} identify its leading population term as $\nu_1 A^{(2)}$; sample concentration gives the following bound.
 
\begin{theorem}\label{thm:matrix_conc_2}
Consider the idealized estimator $\tilde C^{(2)}_2
\;:=\;
\frac{1}{n}
\sum_{\mu=1}^n
y_\mu \;
H_2\!\left(\bh^{(1)}_\mu\right)$ obtained by replacing $\widehat \bh^{(1)}_\mu$ with the true features $\bh^{(1)}_\mu$.
    Under Assumptions~\ref{ass:1}--\ref{ass:2}, with fixed $k\geq1$, $d_1\asymp d^\varepsilon$ for $0<\varepsilon<1/3$, and polynomially growing $n\gtrsim d_1$, we have w.h.p over the teacher parameters and samples as $d\rightarrow\infty$:
    \begin{equation}
    \begin{split}
       &\sqrt{d_1}\|\tilde C^{(2)}_2-\nu_1A^{(2)}\|_2 \\
       &=\tilde{O}\left(\frac{d_1^{3/2}}{\sqrt{d}}\right)+ \tilde{O}\left(\sqrt{\frac{d^{2}_1}{n}}\right)+\tilde{O}\left(\sqrt{\frac{1}{d_1}}\right)
    \end{split}
    \end{equation} 
\end{theorem}
An analogous bound in the recovered coordinates, up to their rotation ambiguity, remains conjectural; see Remark~\ref{rem:tensor}.

\begin{figure}[t]
   \centering
\includegraphics[width=0.9\linewidth]{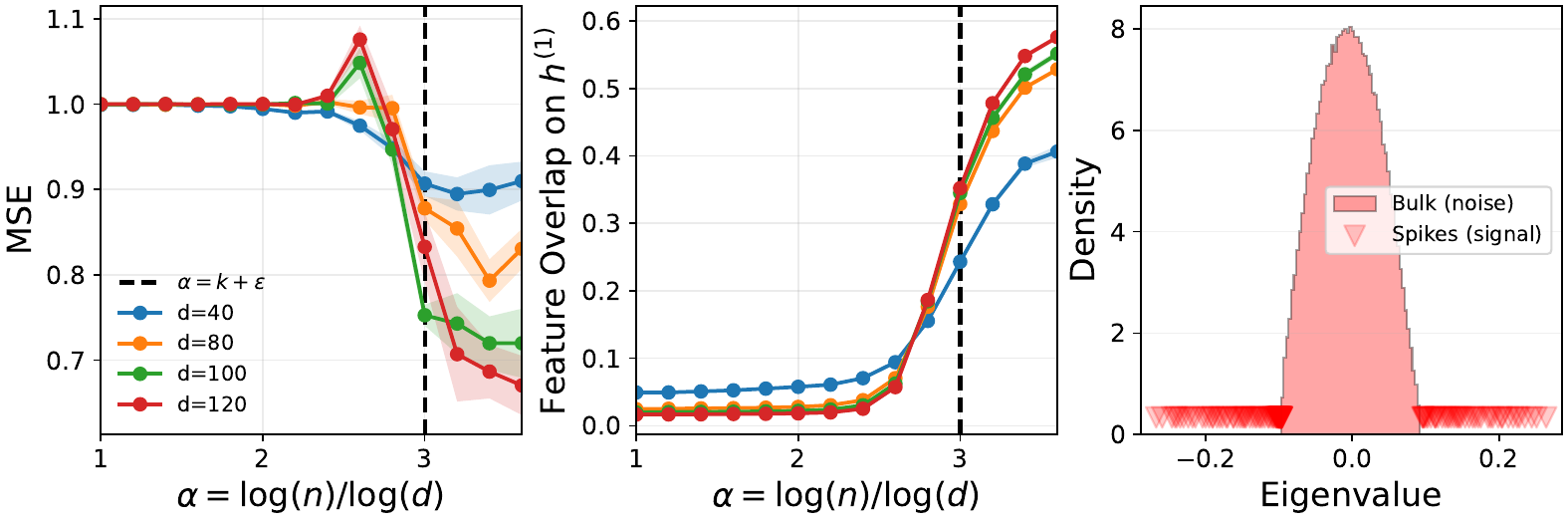}
   \caption{\textbf{On the role of $d^\epsilon$.} Performance of Algorithm~\ref{alg:agnostic-recovery} on \eqref{eq:main:simple_target} with $d_1=d$, or $\varepsilon=1$. MSE and feature overlap are plotted against $\alpha=\log_d n$ for $d\in\{40,80,100,120\}$. The spectrum is shown for $d=120$.}
   \label{fig:second_fig_subplots}
\end{figure}

\paragraph{Fitting $g^\star$.} The final step fits a one-dimensional kernel regressor to $\{(\widehat h_\mu^{(2)},y_\mu)\}_{\mu=1}^n$ \cite{vapnik1998statistical}. This adds no ambient-dimensional exponent to the sample-complexity prediction.

\paragraph{Examples and illustrations.}
For the target in \eqref{eq:main:simple_target} with $g^\star(h)=h$, the quartic regression baseline has sample scale $n=\mathcal O(d^4)$. The hierarchical spectral estimator proceeds in two quadratic steps.
At the first stage, it forms the moment matrix
\begin{equation}
\label{eq:main:chat1}
\widehat C^{(1)}_2
=
\frac{1}{n}
\sum_{\mu=1}^n
y_\mu \;
H_2\!\big(F[H_2(\bx_\mu)]\big),
\end{equation}
which acts on the $\mathcal{O}(d^2)$-dimensional space of symmetric matrices.  
For $0<\varepsilon<1/3$, Theorems~\ref{thm:matrix_conc_1} and~\ref{thm:matrix_conc_2} give sample scales $n\gg d^{2+\varepsilon}$ for the projected first-layer moment and $n\gg d^{2\varepsilon}$ for idealized second-layer estimation, up to logarithmic factors. Under the recovery conjecture, these suggest the overall threshold
\[
n = \mathcal{O}\!\left(\max\{d^{2+\epsilon},\, d^{2\varepsilon}\}\right),
\]
below the quartic baseline $\mathcal O(d^4)$ when $\varepsilon<2$. At $\varepsilon=1/2$, this predicts $\alpha=2.5$, consistent with Fig.~\ref{fig:main_fig_subplots} and outside the proved range. The MSE also exhibits an interpolation peak characteristic of double descent \cite{Belkin_2019,mei2022generalization}.
\medskip

The experiments also test $d_1=d$ (Fig.~\ref{fig:second_fig_subplots}), a $\tanh$ readout (Fig.~\ref{fig:eps05_gtanh}), and higher Hermite degrees (Appendix~\ref{sec:app:numerics}). The observed spectral alignment supports the recovery conjecture in Remark~\ref{rem:tensor} beyond the proved dimension range.

\section{Gaussian Equivalence}\label{sec:HGEM}
Theorems~\ref{thm:matrix_conc_1},~\ref{thm:matrix_conc_2} rely on a quantitative joint Gaussian approximation for the latent projections of $H_k(\bx)$ along the teacher tensors.
   \begin{figure}
   \centering
\includegraphics[width=0.9\linewidth]{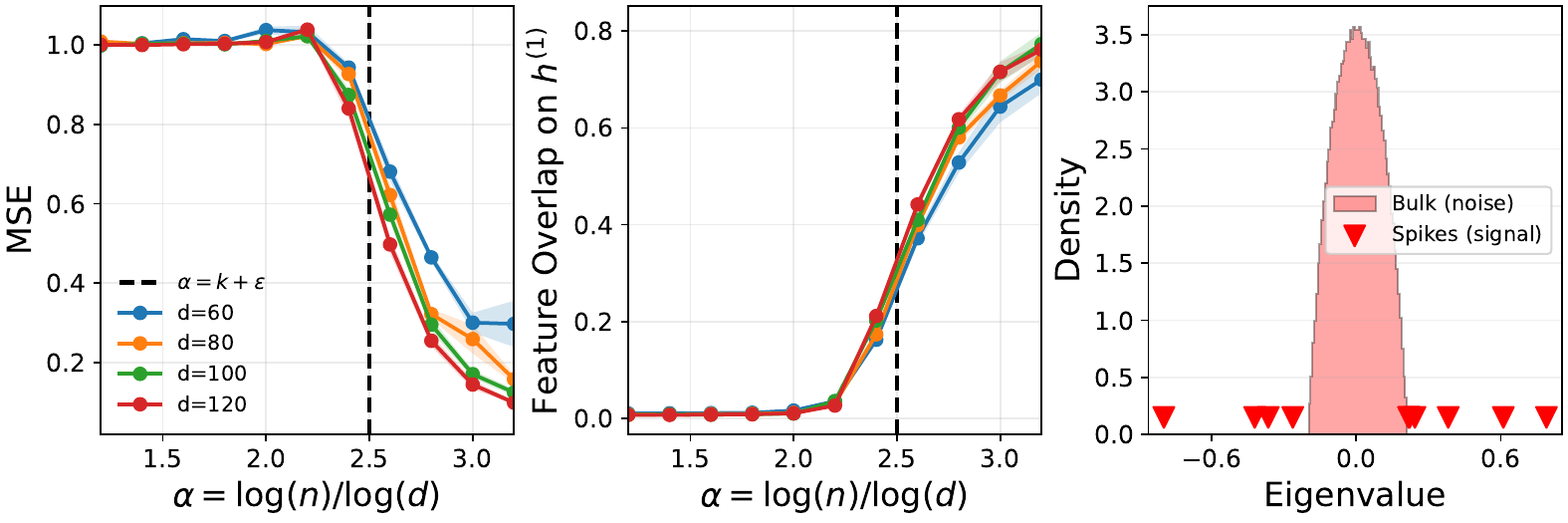}
   \caption{\textbf{On the role of $g^\star$.}
The experiment in Fig.~\ref{fig:main_fig_subplots} with $g^\star=\tanh$, for $d\in\{60,80,100,120\}$. MSE and feature overlap are plotted against $\alpha=\log_d n$; the spectrum uses $d=120$. The qualitative transition persists.
}
   \label{fig:eps05_gtanh}
\end{figure}

\paragraph{One-dimensional CLT.}
The following scalar comparison follows from the Wiener-chaos CLT \cite{nualart2005central,nourdin2009stein}; see also Theorem 4.1 of \cite{wen2025does}.
\begin{lemma}\label{lem:1d-clt}
For any $\kappa,c>0$ and bounded Lipschitz $\psi: \mathbb{R} \rightarrow \mathbb{R}$, we have:
\begin{equation}
  \sup_{\substack{
T : \|T \otimes_r T\|_F\!\le\! \frac{\kappa}{d^c} \\
\forall r \in [k-1]
}}\!\!\!\!\!\! {\mathbb{E}[\psi(\langle T, H_k(\bx)\rangle)] \!-\! \mathbb{E}[\psi(\langle F(T), \tilde{\bx}\rangle)]} \!\!\xrightarrow[d \to \infty]{P} \!0,
\end{equation}
where $\tilde{\bx} \sim \mathcal{N}(0,I_{D_1})$ and $F(T)$ denotes the flattened version of the tensor $T$ and $T \otimes_r T$ denotes the contraction of order $r$.
\end{lemma}

Lemma~\ref{lem:1d-clt} ensures Gaussian approximation for these scalar projections of $H_k(\bx)$. For the random tensors in Assumption~\ref{ass:1}, Lemma~\ref{lem:1dclt} controls the conditional joint law of $r$ features $\langle A_1,H_k(\bx)\rangle,\ldots,\langle A_r,H_k(\bx)\rangle$ when $r$ grows polynomially and $r=o(d^k)$. It gives Wasserstein error $\tilde{\mathcal O}(r/\sqrt d)$ and a quantitative Gram-matrix bound on a common high-probability event over the tensors, using the approach in \cite{nourdin2010multivariate}.

To identify the signal in Gaussian latent coordinates, we compute the label's leading second-order Hermite coefficient.

\begin{lemma}[Composition of Hermite features]\label{lem:comp_hermite}
    For $f \in L_2(\gamma)$, let $P_2(f)$ denote the coefficient matrix of the projection of $f$ onto Hermite polynomials of order $2$. Let $\mathbf{u} \in \mathbb{R}^d$ and symmetric $A \in \mathbb{R}^{d \times d}$  with $\norm{A}_F^2=1+\tilde{O}_d(d^{-1/2})$, $\norm{A}_2= \mathcal{O}(\frac{1}{\sqrt{d}})$. For $g^\star$ satisfying Assumption \ref{ass:2}, consider the quadratic feature $p(\mathbf{u})=\langle A,H_2(\mathbf{u})\rangle$ and the mapping $\mathcal{F}(\mathbf{u}) \coloneqq g^\star(p(\mathbf{u}))$. We have:
    \begin{equation}
        \sqrt{d}\norm{P_2(\mathcal{F})-\nu_1 A}_2=\tilde{O}_d\left(\frac{1}{\sqrt{d}}\right)
    \end{equation}
\end{lemma}
The proof is provided in Appendix \ref{app:comp-hermite-proof}.

Lemma~\ref{lem:latent-moment-comparison} transfers the latent second moment to Gaussian coordinates, and Lemma~\ref{lem:comp_hermite} identifies its leading term as $\nu_1 A^{(2)}$. For Theorem~\ref{thm:matrix_conc_1}, the identity \eqref{eq:latent-centering} relates this latent moment to the projected expectation $A^{(1)}\mathbb E[\hat C_k^{(1)}]A^{(1)\top}$; its Gram centering correction is controlled by Lemma~\ref{lem:teacher-moments}. For Theorem~\ref{thm:matrix_conc_2}, the idealized estimator's expectation is already the latent second moment.

\section{Discussion and Limitations}

The model isolates how an intermediate representation can reduce a high-degree feature-learning problem to successive lower-order spectral estimates.

Our analysis assumes Gaussian inputs and explicit Hermite features. Replacing the Hermite coordinates with nonlinear random features would reduce the method's dependence on the teacher's structure. Hierarchical polynomial learning~\cite{nichani2024provable,wang2023learning,fu2025learning} and Neural Low-Degree Filtering~\cite{dandi2026lofi} suggest possible approaches.

The proved bounds require $0<\varepsilon<1/3$, owing to the error $\tilde{\mathcal O}(d_1^{3/2}/\sqrt d)$. Full spectral recovery and the use of estimated first-layer features remain conjectural, although the experiments support the predicted transition beyond this range. Other open directions include deeper hierarchies, more general nonlinearities, power-law settings, and a rigorous connection with gradient-based training.

\section*{Acknowledgments}
The authors would like to thank Joan Bruna, Alex Damian, Jason Lee, Yue Lu, Theodor Misiakiewicz and Lenka Zdeborova for helpful discussion and feedback. We acknowledge funding from the Swiss National Science Foundation grants  OperaGOST (grant number $200021\ 200390$), DSGIANGO (grant number $225837$), and from the Simons Collaboration on the Physics of Learning and Neural Computation via the Simons Foundation grant ($\#1257412$).

\begingroup
\setlength{\bibsep}{2pt plus 1pt minus 1pt}
\bibliographystyle{plainnat}
\bibliography{refs_arXiv}

@article{nualart2005central,
  title={Central limit theorems for sequences of multiple stochastic integrals},
  author={Nualart, David and Peccati, Giovanni},
  journal={The Annals of Probability},
  volume={33},
  number={1},
  pages={177--193},
  year={2005},
  doi={10.1214/009117904000000621}
}

@article{gardner1989three,
  title={Three unfinished works on the optimal storage capacity of networks},
  author={Gardner, Elizabeth and Derrida, Bernard},
  journal={Journal of Physics A: Mathematical and General},
  volume={22},
  number={12},
  pages={1983--1994},
  year={1989}
}

@article{nourdin2010multivariate,
  title={Multivariate normal approximation using Stein's method and Malliavin calculus},
  author={Nourdin, Ivan and Peccati, Giovanni and R{\'e}veillac, Anthony},
  journal={Annales de l'I.H.P. Probabilit{\'e}s et statistiques},
  volume={46},
  number={1},
  pages={45--58},
  year={2010},
  doi={10.1214/08-AIHP308}
}

@article{nourdin2009stein,
  title={Stein’s method on Wiener chaos},
  author={Nourdin, Ivan and Peccati, Giovanni},
  journal={Probability Theory and Related Fields},
  volume={145},
  number={1},
  pages={75--118},
  year={2009},
  publisher={Springer}
}

@article{wen2025does,
  title={When does Gaussian equivalence fail and how to fix it: Non-universal behavior of random features with quadratic scaling},
  author={Wen, Garrett G and Hu, Hong and Lu, Yue M and Fan, Zhou and Misiakiewicz, Theodor},
  journal={arXiv preprint arXiv:2512.03325},
  year={2025}
}

@article{nichani2024provable,
  title={Provable guarantees for nonlinear feature learning in three-layer neural networks},
  author={Nichani, Eshaan and Damian, Alex and Lee, Jason D},
  journal={Advances in Neural Information Processing Systems},
  volume={36},
  year={2024}
}

@article{allen2019can,
  title={What Can ResNet Learn Efficiently, Going Beyond Kernels?},
  author={Allen-Zhu, Zeyuan and Li, Yuanzhi},
  journal={Advances in Neural Information Processing Systems},
  volume={32},
  year={2019}
}

@article{poggio2017and,
  title={Why and when can deep-but not shallow-networks avoid the curse of dimensionality: a review},
  author={Poggio, Tomaso and Mhaskar, Hrushikesh and Rosasco, Lorenzo and Miranda, Brando and Liao, Qianli},
  journal={International Journal of Automation and Computing},
  volume={14},
  number={5},
  pages={503--519},
  year={2017},
  publisher={Springer}
}

@inproceedings{fu2025learning,
  title={Learning Hierarchical Polynomials of Multiple Nonlinear Features with Three-Layer Networks},
  author={Fu, Hengyu and Wang, Zihao and Nichani, Eshaan and Lee, Jason D.},
  booktitle={Proceedings of the International Conference on Learning Representations (ICLR)},
  year={2025}
}

@inproceedings{montanari2022universality,
  title={Universality of empirical risk minimization},
  author={Montanari, Andrea and Saeed, Basil N},
  booktitle={Conference on Learning Theory},
  pages={4310--4312},
  year={2022},
  organization={PMLR}
}

@inproceedings{daniely2017depth,
  title={Depth separation for neural networks},
  author={Daniely, Amit},
  booktitle={Conference on Learning Theory},
  pages={690--696},
  year={2017},
  organization={PMLR}
}

@inproceedings{allen2023backward,
  title={Backward feature correction: How deep learning performs deep (hierarchical) learning},
  author={Allen-Zhu, Zeyuan and Li, Yuanzhi},
  booktitle={The Thirty Sixth Annual Conference on Learning Theory},
  pages={4598--4598},
  year={2023},
  organization={PMLR}
}

@InProceedings{pmlr-v49-telgarsky16,
  title = 	 {benefits of depth in neural networks},
  author = 	 {Telgarsky, Matus},
  booktitle = 	 {29th Annual Conference on Learning Theory},
  pages = 	 {1517--1539},
  year = 	 {2016},
  editor = 	 {Feldman, Vitaly and Rakhlin, Alexander and Shamir, Ohad},
  volume = 	 {49},
  series = 	 {Proceedings of Machine Learning Research},
  address = 	 {Columbia University, New York, New York, USA},
  month = 	 {23--26 Jun},
  publisher =    {PMLR}
}

@article{bonnaire2025role,
  title={The role of the time-dependent Hessian in high-dimensional optimization},
  author={Bonnaire, Tony and Biroli, Giulio and Cammarota, Chiara},
  journal={Journal of Statistical Mechanics: Theory and Experiment},
  volume={2025},
  number={8},
  pages={083401},
  year={2025},
  publisher={IOP Publishing}
}

@inproceedings{gerace2020generalisation,
  title={Generalisation error in learning with random features and the hidden manifold model},
  author={Gerace, Federica and Loureiro, Bruno and Krzakala, Florent and M{\'e}zard, Marc and Zdeborov{\'a}, Lenka},
  booktitle={International Conference on Machine Learning},
  pages={3452--3462},
  year={2020},
  organization={PMLR}
}

@article{el2008spectrum,
  title={Spectrum estimation for large dimensional covariance matrices using random matrix theory},
  author={El Karoui, Noureddine},
  journal={The Annals of Statistics},
  volume={36},
  number={6},
  year={2008},
  publisher={Institute of Mathematical Statistics}
}

@InProceedings{COLTXU2025,
  title = 	 {Fundamental Limits of Matrix Sensing: Exact Asymptotics, Universality, and Applications},
  author =       {Xu, Yizhou and Maillard, Antoine and Zdeborov\'a, Lenka and Krzakala, Florent},
  booktitle = 	 {Proceedings of Thirty Eighth Conference on Learning Theory},
  pages = 	 {5757--5823},
  year = 	 {2025},
  editor = 	 {Haghtalab, Nika and Moitra, Ankur},
  volume = 	 {291},
  series = 	 {Proceedings of Machine Learning Research},
  month = 	 {30 Jun--04 Jul},
  publisher =    {PMLR}
}

@article{lu2025equivalence,
  title={An equivalence principle for the spectrum of random inner-product kernel matrices with polynomial scalings},
  author={Lu, Yue M and Yau, Horng-Tzer},
  journal={The Annals of Applied Probability},
  volume={35},
  number={4},
  pages={2411--2470},
  year={2025},
  publisher={Institute of Mathematical Statistics}
}

@article{bandeira2025exact,
  title={Exact threshold for approximate ellipsoid fitting of random points},
  author={Bandeira, Afonso S and Maillard, Antoine},
  journal={Electronic Journal of Probability},
  volume={30},
  pages={1--46},
  year={2025},
  publisher={The Institute of Mathematical Statistics and the Bernoulli Society}
}

@article{hu2024asymptotics,
  title={Asymptotics of random feature regression beyond the linear scaling regime},
  author={Hu, Hong and Lu, Yue M and Misiakiewicz, Theodor},
  journal={arXiv preprint arXiv:2403.08160},
  year={2024}
}

@article{dandi2023universality,
  title={Universality laws for gaussian mixtures in generalized linear models},
  author={Dandi, Yatin and Stephan, Ludovic and Krzakala, Florent and Loureiro, Bruno and Zdeborov{\'a}, Lenka},
  journal={Advances in Neural Information Processing Systems},
  volume={36},
  pages={54754--54768},
  year={2023}
}

@inproceedings{defilippis2024dimension,
  title={Dimension-Free Deterministic Equivalents and Scaling Laws for Random Feature Regression},
  author={Defilippis, Leonardo and Loureiro, Bruno and Misiakiewicz, Theodor},
  booktitle={Advances in Neural Information Processing Systems},
  year={2024}
}

@article{xiao2022precise,
  title={Precise learning curves and higher-order scalings for dot-product kernel regression},
  author={Xiao, Lechao and Hu, Hong and Misiakiewicz, Theodor and Lu, Yue and Pennington, Jeffrey},
  journal={Advances in Neural Information Processing Systems},
  volume={35},
  pages={4558--4570},
  year={2022}
}

@article{sejnowski2020unreasonable,
  title={The unreasonable effectiveness of deep learning in artificial intelligence},
  author={Sejnowski, Terrence J},
  journal={Proceedings of the National Academy of Sciences},
  volume={117},
  number={48},
  year={2020},
  publisher={National Acad Sciences}
}

@inproceedings{mhaskar2017and,
  title={When and why are deep networks better than shallow ones?},
  author={Mhaskar, Hrushikesh and Liao, Qianli and Poggio, Tomaso},
  booktitle={Proceedings of the AAAI conference on artificial intelligence},
  volume={31},
  year={2017}
}

@article{aubin2018committee,
  title={The committee machine: Computational to statistical gaps in learning a two-layers neural network},
  author={Aubin, Benjamin and Maillard, Antoine and Krzakala, Florent and Macris, Nicolas and Zdeborov{\'a}, Lenka and others},
  journal={Advances in Neural Information Processing Systems},
  volume={31},
  year={2018}
}

@article{barbier2019optimal,
  title={Optimal errors and phase transitions in high-dimensional generalized linear models},
  author={Barbier, Jean and Krzakala, Florent and Macris, Nicolas and Miolane, L{\'e}o and Zdeborov{\'a}, Lenka},
  journal={Proceedings of the National Academy of Sciences},
  volume={116},
  number={12},
  pages={5451--5460},
  year={2019},
  publisher={National Acad Sciences}
}

@inproceedings{maillard2022construction,
  title={Construction of optimal spectral methods in phase retrieval},
  author={Maillard, Antoine and Krzakala, Florent and Lu, Yue M and Zdeborov{\'a}, Lenka},
  booktitle={Mathematical and Scientific Machine Learning},
  pages={693--720},
  year={2022},
  organization={PMLR}
}

@article{lu2020phase,
  title={Phase transitions of spectral initialization for high-dimensional non-convex estimation},
  author={Lu, Yue M and Li, Gen},
  journal={Information and Inference: A Journal of the IMA},
  volume={9},
  number={3},
  pages={507--541},
  year={2020},
  publisher={Oxford University Press}
}

@inproceedings{mondelli2018fundamental,
  title={Fundamental limits of weak recovery with applications to phase retrieval},
  author={Mondelli, Marco and Montanari, Andrea},
  booktitle={Conference On Learning Theory},
  pages={1445--1450},
  year={2018},
  organization={PMLR}
}

@inproceedings{dandicomputational,
  title={The Computational Advantage of Depth in Learning High-Dimensional Hierarchical Targets},
  author={Dandi, Yatin and Pesce, Luca and Zdeborova, Lenka and Krzakala, Florent},
  booktitle={The Thirty-ninth Annual Conference on Neural Information Processing Systems},
year = {2025}
}

@article{mossel2016deep,
  title={Deep learning and hierarchal generative models},
  author={Mossel, Elchanan},
  journal={arXiv preprint arXiv:1612.09057},
  year={2016}
}

@inproceedings{cagnetta2024towards,
  title={Towards a Theory of How the Structure of Language is Acquired by Deep Neural Networks},
  author={Cagnetta, Francesco and Wyart, Matthieu},
  booktitle={Advances in Neural Information Processing Systems},
  year={2024}
}

@article{mehta2014exact,
  title={An exact mapping between the variational renormalization group and deep learning},
  author={Mehta, Pankaj and Schwab, David J},
  journal={arXiv preprint arXiv:1410.3831},
  year={2014}
}

@article{wilson1971renormalization,
  title={Renormalization group and critical phenomena. II. Phase-space cell analysis of critical behavior},
  author={Wilson, Kenneth G},
  journal={Physical Review B},
  volume={4},
  number={9},
  pages={3184},
  year={1971},
  publisher={APS}
}

@article{li2018neural,
  title={Neural network renormalization group},
  author={Li, Shuo-Hui and Wang, Lei},
  journal={Physical review letters},
  volume={121},
  number={26},
  pages={260601},
  year={2018},
  publisher={APS}
}

@article{marchand2023multiscale,
  title={Multiscale Data-Driven Energy Estimation and Generation},
  author={Marchand, Tanguy and Ozawa, Misaki and Biroli, Giulio and Mallat, Stéphane},
  journal={Physical Review X},
  volume={13},
  number={4},
  year={2023},
  publisher={American Physical Society}
}

@inproceedings{ba2020generalization,
  title={Generalization of two-layer neural networks: An asymptotic viewpoint},
  author={Ba, Jimmy and Erdogdu, Murat and Suzuki, Taiji and Wu, Denny and Zhang, Tianzong},
  booktitle={International conference on learning representations},
  year={2020}
}

@inproceedings{damian2022neural,
  title={Neural networks can learn representations with gradient descent},
  author={Damian, Alexandru and Lee, Jason and Soltanolkotabi, Mahdi},
  booktitle={Conference on Learning Theory},
  pages={5413--5452},
  year={2022},
  organization={PMLR}
}

@article{ghorbani2021linearized,
  title={Linearized two-layers neural networks in high dimension},
  author={Ghorbani, Behrooz and Mei, Song and Misiakiewicz, Theodor and Montanari, Andrea},
  journal={The Annals of Statistics},
  volume={49},
  number={2},
  year={2021}
}

@article{ghorbani2020neural,
  title={When do neural networks outperform kernel methods?},
  author={Ghorbani, Behrooz and Mei, Song and Misiakiewicz, Theodor and Montanari, Andrea},
  journal={Advances in Neural Information Processing Systems},
  volume={33},
  pages={14820--14830},
  year={2020}
}

@article{mei2022generalization,
  title={Generalization error of random feature and kernel methods: hypercontractivity and kernel matrix concentration},
  author={Mei, Song and Misiakiewicz, Theodor and Montanari, Andrea},
  journal={Applied and Computational Harmonic Analysis},
  volume={59},
  pages={3--84},
  year={2022},
  publisher={Elsevier}
}

@article{wang2023learning,
  title={Learning hierarchical polynomials with three-layer neural networks},
  author={Wang, Zihao and Nichani, Eshaan and Lee, Jason D},
  journal={arXiv preprint arXiv:2311.13774},
  year={2023}
}

@article{baik2005phase,
  title={Phase transition of the largest eigenvalue for nonnull complex sample covariance matrices},
  author={Baik, Jinho and Ben Arous, G{\'e}rard and P{\'e}ch{\'e}, Sandrine},
  journal={The Annals of Probability},
  volume={33},
  number={5},
  pages={1643--1697},
  year={2005},
  doi={10.1214/009117905000000233}
}

@article{dhifallah2020precise,
  title={A precise performance analysis of learning with random features},
  author={Dhifallah, Oussama and Lu, Yue M},
  journal={arXiv preprint arXiv:2008.11904},
  year={2020}
}

@inproceedings{goldt_gaussian_2021,
  title     = {The Gaussian equivalence of generative models for learning with shallow neural networks},
  author    = {Goldt, Sebastian and Loureiro, Bruno and Reeves, Galen and Krzakala, Florent and Mezard, Marc and Zdeborova, Lenka},
  year      = 2022,
  booktitle = {Proceedings of the 2nd Mathematical and Scientific Machine Learning Conference},
  pages     = {426--471}
}

@article{mei_generalization_2022,
  title   = {The Generalization Error of Random Features Regression: Precise Asymptotics and the Double Descent Curve},
  author  = {Mei, Song and Montanari, Andrea},
  year    = 2022,
  journal = {Communications on Pure and Applied Mathematics},
  volume  = 75,
  number  = 4,
  pages   = {667--766},
  eprint  = {https://onlinelibrary.wiley.com/doi/pdf/10.1002/cpa.22008}
}

@article{rahimi2007random,
  title={Random features for large-scale kernel machines},
  author={Rahimi, Ali and Recht, Benjamin},
  journal={Advances in neural information processing systems},
  volume={20},
  year={2007}
}

@article{BenArous2021,
  author  = {Gerard {Ben Arous} and Reza Gheissari and Aukosh Jagannath},
  title   = {Online stochastic gradient descent on non-convex losses from high-dimensional inference},
  journal = {Journal of Machine Learning Research},
  year    = {2021},
  volume  = {22},
  number  = {106},
  pages   = {1--51},
}

@article{troiani2024fundamental,
  title={Fundamental limits of weak learnability in high-dimensional multi-index models},
  author={Troiani, Emanuele and Dandi, Yatin and Defilippis, Leonardo and Zdeborov{\'a}, Lenka and Loureiro, Bruno and Krzakala, Florent},
  journal={arXiv preprint arXiv:2405.15480},
  year={2024}
}

@article{bietti2022learning,
  title={Learning single-index models with shallow neural networks},
  author={Bietti, Alberto and Bruna, Joan and Sanford, Clayton and Song, Min Jae},
  journal={Advances in Neural Information Processing Systems},
  volume={35},
  pages={9768--9783},
  year={2022}
}

@inproceedings{abbe2023sgd,
  title={Sgd learning on neural networks: leap complexity and saddle-to-saddle dynamics},
  author={Abbe, Emmanuel and Adsera, Enric Boix and Misiakiewicz, Theodor},
  booktitle={The Thirty Sixth Annual Conference on Learning Theory},
  pages={2552--2623},
  year={2023},
  organization={PMLR}
}

@inproceedings{abbe2022merged,
  title        = {The merged-staircase property: a necessary and nearly sufficient condition for sgd learning of sparse functions on two-layer neural networks},
  author       = {Abbe, Emmanuel and Boix-Adsera, Enric  and Misiakiewicz, Theodor},
  booktitle    = {Conference on Learning Theory},
  pages        = {4782--4887},
  year         = {2022},
  organization = {PMLR}
}

@inproceedings{arora2019convergence,
  title={A Convergence Analysis of Gradient Descent for Deep Linear Neural Networks},
  author={Arora, Sanjeev and Cohen, Nadav and Golowich, Noah and Hu, Wei},
  booktitle={7th International Conference on Learning Representations (ICLR)},
  year={2019}
}

@article{lee2019wide,
  title={Wide neural networks of any depth evolve as linear models under gradient descent},
  author={Lee, Jaehoon and Xiao, Lechao and Schoenholz, Samuel and Bahri, Yasaman and Novak, Roman and Sohl-Dickstein, Jascha and Pennington, Jeffrey},
  journal={Advances in neural information processing systems},
  volume={32},
  year={2019}
}

@inproceedings{ji2019gradient,
  title={Gradient Descent Aligns the Layers of Deep Linear Networks},
  author={Ji, Ziwei and Telgarsky, Matus},
  booktitle={Proceedings of the International Conference on Learning Representations (ICLR)},
  year={2019}
}

@article{cagnettarandomhierarchy,
  title = {How Deep Neural Networks Learn Compositional Data: The Random Hierarchy Model},
  author = {Cagnetta, Francesco and Petrini, Leonardo and Tomasini, Umberto M. and Favero, Alessandro and Wyart, Matthieu},
  journal = {Phys. Rev. X},
  volume = {14},
  issue = {3},
  pages = {031001},
  numpages = {24},
  year = {2024},
  month = {Jul},
  publisher = {American Physical Society},
  doi = {10.1103/PhysRevX.14.031001}
}

@inproceedings{saxe2014exact,
  title={Exact solutions to the nonlinear dynamics of learning in deep linear neural networks},
  author={Saxe, Andrew M. and McClelland, James L. and Ganguli, Surya},
  booktitle={Proceedings of the International Conference on Learning Representations (ICLR)},
  year={2014}
}

@misc{zhang2025neuralnetworks,
      title={Neural Networks Learn Generic Multi-Index Models Near Information-Theoretic Limit}, 
      author={Bohan Zhang and Zihao Wang and Hengyu Fu and Jason D. Lee},
      year={2025},
      eprint={2511.15120},
      archivePrefix={arXiv},
      primaryClass={stat.ML},
      url={https://arxiv.org/abs/2511.15120}, 
}

@article{zhang2021understanding,
  title={Understanding deep learning (still) requires rethinking generalization},
  author={Zhang, Chiyuan and Bengio, Samy and Hardt, Moritz and Recht, Benjamin and Vinyals, Oriol},
  journal={Communications of the ACM},
  volume={64},
  number={3},
  pages={107--115},
  year={2021},
  publisher={ACM New York, NY, USA}
}

@article{dandi2024twolayer,
  author  = {Yatin Dandi and Florent Krzakala and Bruno Loureiro and Luca Pesce and Ludovic Stephan},
  title   = {How Two-Layer Neural Networks Learn, One (Giant) Step at a Time},
  journal = {Journal of Machine Learning Research},
  year    = {2024},
  volume  = {25},
  number  = {349},
  pages   = {1--65}
}

@inproceedings{damian2024computational,
  title={Computational-Statistical Gaps in Gaussian Single-Index Models},
  author={Damian, Alex and Pillaud-Vivien, Loucas and Lee, Jason D. and Bruna, Joan},
  booktitle={Proceedings of the 37th Annual Conference on Learning Theory (COLT)},
  year={2024}
}

@article{arnaboldi2024repetita,
  title={Repetita iuvant: Data repetition allows sgd to learn high-dimensional multi-index functions},
  author={Arnaboldi, Luca and Dandi, Yatin and Krzakala, Florent and Pesce, Luca and Stephan, Ludovic},
  journal={arXiv preprint arXiv:2405.15459},
  year={2024}
}

@inproceedings{lee2024neural,
  title={Neural Network Learns Low-Dimensional Polynomials with SGD Near the Information-Theoretic Limit},
  author={Lee, Jason D. and Oko, Kazusato and Suzuki, Taiji and Wu, Denny},
  booktitle={Advances in Neural Information Processing Systems},
  year={2024}
}

@article{kovavcevic2025spectral,
  title={Spectral estimators for multi-index models: Precise asymptotics and optimal weak recovery},
  author={Kova{\v{c}}evi{\'c}, Filip and Zhang, Yihan and Mondelli, Marco},
  journal={arXiv preprint arXiv:2502.01583},
  year={2025}
}

@article{defilippis2025optimal,
  title={Optimal spectral transitions in high-dimensional multi-index models},
  author={Defilippis, Leonardo and Dandi, Yatin and Mergny, Pierre and Krzakala, Florent and Loureiro, Bruno},
  journal={arXiv preprint arXiv:2502.02545},
  year={2025}
}

@book{vapnik1998statistical,
  title     = {Statistical Learning Theory},
  author    = {Vapnik, Vladimir N.},
  year      = {1998},
  publisher = {Wiley},
  address   = {New York}
}

@article{Belkin_2019,
	doi = {10.1073/pnas.1903070116},  
	year = 2019,
	month = {jul},
	publisher = {Proceedings of the National Academy of Sciences},
	volume = {116},
	number = {32},
  
	pages = {15849--15854},
  
	author = {Mikhail Belkin and Daniel Hsu and Siyuan Ma and Soumik Mandal},
  
	title = {Reconciling modern machine-learning practice and the classical bias{\textendash}variance trade-off},
  
	journal = {Proceedings of the National Academy of Sciences}
}

@article{montanari2026phasetransitionsfeaturelearning,
      title={Phase Transitions for Feature Learning in Neural Networks}, 
      author={Andrea Montanari and Zihao Wang},
      journal={arXiv preprint arXiv:2602.01434},
      year={2026}
}

@article{ren2023depthseparationmultilayermeanfield,
      title={Depth Separation with Multilayer Mean-Field Networks}, 
      author={Yunwei Ren and Mo Zhou and Rong Ge},
      journal={arXiv preprint arXiv:2304.01063},
      year={2023}
}

@InProceedings{garnierbrun2025transformerslearnstructureddata,
  title = {How Transformers Learn Structured Data: Insights From Hierarchical Filtering},
  author = {Garnier-Brun, Jerome and Mezard, Marc and Moscato, Emanuele and Saglietti, Luca},
  booktitle = {Proceedings of the 42nd International Conference on Machine Learning},
  pages = {18831--18847},
  year = {2025},
  volume = {267},
  series = {Proceedings of Machine Learning Research},
  month = {13--19 Jul},
  publisher = {PMLR}
}

@article{buzaglo2025uniform,
      title={How Uniform Random Weights Induce Non-uniform Bias: Typical Interpolating Neural Networks Generalize with Narrow Teachers}, 
      author={Gon Buzaglo and Itamar Harel and Mor Shpigel Nacson and Alon Brutzkus and Nathan Srebro and Daniel Soudry},
      journal={arXiv preprint arXiv:2402.06323},
      year={2025},
}

@article{dandi2026lofi,
      title={Deep Learning as Neural Low-Degree Filtering: A Spectral Theory of Hierarchical Feature Learning}, 
      author={Yatin Dandi and Matteo Vilucchio and Luca Arnaboldi and Hugo Tabanelli and Florent Krzakala},
      year={2026},
      journal={arXiv preprint arXiv:2605.13612}
}

@article{vilucchio2025asymptotics,
  title={Asymptotics of non-convex generalized linear models in high-dimensions: A proof of the replica formula},
  author={Vilucchio, Matteo and Dandi, Yatin and Rossignol, Mat{\'e}o Pirio and Gerbelot, Cedric and Krzakala, Florent},
  journal={arXiv preprint arXiv:2502.20003},
  year={2025}
}
\endgroup

\newpage
\appendix

\section{Derivation of theoretical claims}\label{app:derivation}

\subsection{Tensors with vanishing contractions}\label{app:tens_nont}

For symmetric tensors $S\in(\mathbb R^d)^{\odot k}$ and $T\in(\mathbb R^d)^{\odot\ell}$, the contraction of order $r$ is defined by

\[
(S \otimes_r T)_{i_1 \dots i_{k-r}\, j_1 \dots j_{\ell-r}}
=
\sum_{a_1,\dots,a_r = 1}^d
S_{i_1 \dots i_{k-r}\, a_1 \dots a_r}\,
T_{j_1 \dots j_{\ell-r}\, a_1 \dots a_r}.
\]
Its symmetrization over all remaining indices is denoted by $S\widetilde\otimes_r T$; symmetrization cannot increase the Frobenius norm.

Lemma~\ref{lem:1d-clt} requires the Frobenius norms of the non-trivial contractions, of orders $1,\ldots,k-1$, to vanish as $d\to\infty$. For $k=2$ and bounded Frobenius norm, this is equivalent to a vanishing operator norm. These conditions hold for normalized Gaussian tensors and tensors sampled uniformly from a Frobenius sphere (Lemma 5 in \cite{wen2025does}). The joint Gaussian approximation below also uses bounds on cross-contractions between the tensors $A_1,\ldots,A_r$.

\subsection{Non-asymptotic joint-CLT for Wiener Chaos}

Recall the definition of the $1$-Wasserstein distance (Kantorovich-Rubinstein) on $\mathbb R^r$ associated with the Euclidean norm:
\[
d_W(G,Z) := \sup\Big\{ \big|\mathbb E[h(G)]-\mathbb E[h(Z)]\big| : h:\mathbb R^r\to\mathbb R,\ \mathrm{Lip}(h)\le 1 \Big\}.
\]

\begin{lemma}\label{lem:1dclt}
Fix $k\geq 1$ and let $r=r(d)\leq d^M$ for a fixed $M>0$, with $r=o(d^k)$.
Let $A_1,\ldots,A_r$ be independent symmetric tensors of order $k$ whose entries indexed by distinct unordered $k$-tuples are independent $\mathcal N(0,d^{-k})$ variables; entries related by a permutation are equal.
For $\bx\sim\mathcal N(0,I_d)$ independent of these tensors, write
\[
 U_i=\langle A_i,H_k(\bx)\rangle,\qquad
 U=(U_1,\ldots,U_r),\qquad C_{ij}=\langle A_i,A_j\rangle.
\]
For every $L>0$, there is a constant $c_{k,L,M}<\infty$ such that, for all sufficiently large $d$, with probability at least $1-d^{-L}$ over the tensors, $C$ is positive definite and the following bounds hold simultaneously:
\begin{equation}\label{eq:joint-clt-gram}
 \|C-I_r\|_2
 \leq c_{k,L,M}\left[
 \sqrt{\frac{r+\log d}{d^k}}+\frac{r+\log d}{d^k}\right],
\end{equation}
\[
 \max\{d_W(U,Z_C),d_W(U,Z)\}
 \leq c_{k,L,M}\frac{r\log d}{\sqrt d},
 \qquad Z_C\sim\mathcal N(0,C),\quad Z\sim\mathcal N(0,I_r).
\]
Both Wasserstein distances are conditional on the tensors: their expectations integrate only over $\bx$ and the corresponding Gaussian comparison variable.
\end{lemma}

We first record the product formula and contraction estimates used in its proof.

\begin{lemma}[Normalized Malliavin product formula]\label{lem:maliavin}
For $k,\ell\geq1$, $S\in(\mathbb R^d)^{\odot k}$ and $T\in(\mathbb R^d)^{\odot\ell}$,
\[
\begin{aligned}
 &\bigl(D\langle S,H_k(\bx)\rangle\bigr)^\top
       \bigl(D\langle T,H_\ell(\bx)\rangle\bigr)\\
 &\quad=k\ell\sum_{s=1}^{\min(k,\ell)}(s-1)!
 \binom{k-1}{s-1}\binom{\ell-1}{s-1}
 \sqrt{\frac{(k+\ell-2s)!}{k!\ell!}}\,
 \bigl\langle H_{k+\ell-2s}(\bx),\operatorname{Sym}(S\otimes_s T)\bigr\rangle.
\end{aligned}
\]
Here $D$ is the Malliavin derivative, $\otimes_s$ is the raw contraction defined above, and $\operatorname{Sym}$ averages over permutations of the remaining indices.
\end{lemma}
\begin{proof}
This is the usual product formula for multiple Wiener integrals, as in \cite{nualart2005central}, with the normalization
$I_q(S)=\sqrt{q!}\langle S,H_q\rangle$. In the notation of Lemma~\ref{lem:1dclt}, its constant term gives
$\mathbb E_{\bx}[\langle DU_i,DU_j\rangle]=kC_{ij}$.
\end{proof}

\begin{lemma}\label{lem:contract-scaling}
Let $A,B$ be independent symmetric Gaussian tensors with the entry convention of Lemma~\ref{lem:1dclt}. For $1\leq s\leq k$,
\[
 \mathbb E\|A\otimes_s B\|_F^2=\Theta_k(d^{-s}).
\]
For $1\leq s\leq k-1$, the self-contractions satisfy
\[
 \mathbb E\|A\otimes_s A\|_F^2
 =\Theta_k\bigl(d^{-s}+d^{-(k-s)}\bigr).
\]
Moreover, $\mathbb E\|A\|_F^2=1$ and
$\mathbb E(\|A\|_F^2-1)^2=\mathcal O_k(d^{-k})$; hence
$\mathbb E\|A\otimes_k A\|_F^2=1+\mathcal O_k(d^{-k})$.

For fixed $k\geq2$, let $A_1,\ldots,A_r$ be independent copies of $A$, where $r\leq d^M$ for fixed $M>0$. For every $L>0$, there is a constant $c_{k,L,M}<\infty$ such that, for all sufficiently large $d$, with probability at least $1-\tfrac12d^{-L}$ over these tensors,
\begin{equation}\label{eq:uniform-contractions}
 \max_{i,j\leq r}\max_{1\leq s<k}
 \|A_i\otimes_s A_j\|_F^2
 \leq c_{k,L,M}\frac{(\log d)^2}{d}.
\end{equation}
\end{lemma}

\begin{proof}
For ordered tuples, write $a\sim b$ when they have the same multiset of indices, and define
\[
 N_j(d)=\#\{(a,b)\in[d]^j\times[d]^j:a\sim b\},\qquad N_0(d)=1.
\]
Then $d^j\leq N_j(d)\leq j!d^j$, and the covariance of tensor entries is
$\mathbb E[A_aA_b]=d^{-k}\mathbf1_{\{a\sim b\}}$.
Expanding
$(A\otimes_s B)_{a,b}=\sum_{m\in[d]^s}A_{a,m}B_{b,m}$
and using independence of $A,B$ gives
\[
 \mathbb E\|A\otimes_s B\|_F^2=d^{-2s}N_s(d),
\]
since $(a,m)\sim(a,n)$ is equivalent to $m\sim n$.
For a self-contraction, the three Gaussian Wick pairings instead give
\[
 \mathbb E\|A\otimes_s A\|_F^2
 =d^{-2k}\left[
 d^{2s}N_{k-s}(d)+d^{2(k-s)}N_s(d)+N_{k-s}(d)N_s(d)
 \right].
\]
For the crossed pairing, the two multiset equalities imply both $a\sim b$ and $m\sim n$, which yields the final term. These formulas prove the nontrivial-contraction estimates. At $s=k$, the latter formula equals
$1+2d^{-2k}N_k(d)$; together with $\mathbb E\|A\|_F^2=1$, this proves the stated moment estimates.

To prove \eqref{eq:uniform-contractions}, note that each coordinate of $A_i\otimes_s A_j$ is a degree-two polynomial in independent standard Gaussian coefficients. Gaussian hypercontractivity and Minkowski's inequality give, for $p\geq2$,
\[
 \left(\mathbb E_A\|A_i\otimes_s A_j\|_F^p\right)^{1/p}
 \leq c_k p\left(\mathbb E_A\|A_i\otimes_s A_j\|_F^2\right)^{1/2}.
\]
Taking $p$ to be a sufficiently large multiple of $\log d$, Markov's inequality and a union bound over the at most $kr^2$ choices give \eqref{eq:uniform-contractions} with probability at least $1-\tfrac12d^{-L}$.
\end{proof}

\begin{proof}[Proof of Lemma~\ref{lem:1dclt}]
We first prove the Gram bound \eqref{eq:joint-clt-gram}.
In an orthonormal basis of symmetric tensors, the coefficient covariance of one row is diagonal, with entries
$\lambda_\alpha=m_\alpha d^{-k}$, where $m_\alpha\leq k!$ is the number of ordered permutations of the unordered tuple $\alpha$. Thus
\[
 \sum_\alpha\lambda_\alpha=1,\qquad
 \max_\alpha\lambda_\alpha\leq k!d^{-k},\qquad
 \sum_\alpha\lambda_\alpha^2\leq k!d^{-k}.
\]
Write $C=G\operatorname{diag}(\lambda_\alpha)G^\top$, where $G$ has independent standard Gaussian entries. For each fixed unit vector, the quadratic form of $C-I_r$ is
$\sum_\alpha\lambda_\alpha(Z_\alpha^2-1)$; its Gaussian moment-generating function gives a tail bounded by $2e^{-t}$ at a constant times
$\sqrt{t\sum_\alpha\lambda_\alpha^2}+t\max_\alpha\lambda_\alpha$.
Taking $t$ of order $r+(L+1)\log d$ and using a $1/4$-net of size at most $9^r$ gives \eqref{eq:joint-clt-gram} with probability at least $1-\tfrac12d^{-L}$.
Because $r=o(d^k)$, it implies $\|C-I_r\|_2\leq1/2$ for all sufficiently large $d$ on this event.

For $k=1$, $U$ is already Gaussian with covariance $C$, so $d_W(U,Z_C)=0$ and only the comparison with $Z$ below is needed.
For $k\geq2$, intersect the Gram event with the event \eqref{eq:uniform-contractions} from Lemma~\ref{lem:contract-scaling}. Their intersection has probability at least $1-d^{-L}$; on it, fix the tensors.
Lemma~\ref{lem:maliavin}, orthogonality of distinct Wiener chaoses, and
$\|\operatorname{Sym}(T)\|_F\leq\|T\|_F$ imply
\[
 \mathbb E_{\bx}\left[
 \left(C_{ij}-\frac1k\langle DU_i,DU_j\rangle\right)^2
 \right]
 \leq c_k\sum_{s=1}^{k-1}\|A_i\otimes_s A_j\|_F^2
 \leq c_{k,L,M}\frac{(\log d)^2}{d}.
\]
Here the final inequality is \eqref{eq:uniform-contractions}. Corollary 3.6 in \cite{nourdin2010multivariate} (and its scalar version when $r=1$) now yields
\[
\begin{aligned}
 d_W(U,Z_C)
 &\leq\|C^{-1}\|_2\|C\|_2^{1/2}
 \left(\sum_{i,j=1}^r\mathbb E_{\bx}
 \left[\left(C_{ij}-\frac1k\langle DU_i,DU_j\rangle\right)^2\right]\right)^{1/2}\\
 &\leq c_{k,L,M}\frac{r\log d}{\sqrt d}.
\end{aligned}
\]
Finally, for every $k\geq1$, coupling $Z_C=C^{1/2}Z$ with $Z\sim\mathcal N(0,I_r)$ and applying \eqref{eq:joint-clt-gram} gives
\[
\begin{aligned}
 d_W(Z_C,Z)
 &\leq\|C^{1/2}-I_r\|_F
 \leq\sqrt r\,\|C-I_r\|_2\\
 &\leq c_{k,L}\frac{r\sqrt{\log d}}{d^{k/2}}
 \leq c_{k,L}\frac{r\log d}{\sqrt d}
\end{aligned}
\]
for all sufficiently large $d$. The triangle inequality completes the proof.
\end{proof}

\subsection{Tail bounds and latent moment comparison}

We use the following norm bound for $H_k(\bx)$.

\begin{lemma}[Lemma  F.4  in \cite{wen2025does}]\label{lem:tail}
     There are constants $c,C,C'>0$ such that, with probability at least $1-Ce^{-cd}$,
     \begin{equation}
         \norm{F[H_k(\bx_\mu)]}^2 \leq C'd^k
     \end{equation}
\end{lemma}

The next lemma collects coefficient and conditional moment bounds on a common high-probability event over the teachers. These bounds control truncation errors in Lemma~\ref{lem:latent-moment-comparison} and the proofs of Theorems~\ref{thm:matrix_conc_1}--\ref{thm:matrix_conc_2}, as well as the Gram centering correction in Theorem~\ref{thm:matrix_conc_1}.

\begin{lemma}[Teacher event and conditional moments]\label{lem:teacher-moments}
Assume the Gaussian teacher model of Assumption~\ref{ass:1}, with fixed
$k\geq1$ and $d_1\asymp d^\varepsilon$ for fixed $0<\varepsilon<1/3$,
and let $g^\star$ satisfy Assumption~\ref{ass:2}. Define
\[
 u=h^{(1)}(\bx),\qquad G=A^{(1)}A^{(1)\top},\qquad
 Q(w)=\langle A^{(2)},H_2(w)\rangle,\qquad
 Z\sim\mathcal N(0,I_{d_1}),
\]
where $\bx\sim\mathcal N(0,I_d)$ and $Z$ are independent of the teachers.
For every fixed $L>0$ and all sufficiently large $d$, there is a teacher
event $\mathcal T_d$ of probability at least $1-d^{-L}$ on which the
following estimates hold. Throughout the conclusions, expectations, probabilities,
and $L^p$ norms are conditional on the teachers and integrate only over
$\bx$ or $Z$, as appropriate.
First,
\[
 d_W(u,Z)\leq C\frac{d_1\log d}{\sqrt d},
\]
and
\begin{equation}\label{eq:teacher-event}
\begin{split}
 \|G-I_{d_1}\|_2
   &=\tilde{\mathcal O}\left(\sqrt{d_1/d^k}\right),
 \qquad \|A^{(2)}\|_2=\mathcal O(d_1^{-1/2}),\\
 \big|\|A^{(2)}\|_F^2-1\big|
   &=\tilde{\mathcal O}(d_1^{-1}),
 \qquad \|Q(u)\|_{L^2}\leq C\sqrt{\log d}.
\end{split}
\end{equation}
Uniformly for all $p\geq2$, $W\in\{u,Z\}$, and unit vectors $\bv$,
\begin{align}
 \|\bv^\top W\|_{L^p}
   &\leq C(p-1)^{k/2},
 &\big\|\|W\|\big\|_{L^p}
   &\leq C\sqrt{d_1}(p-1)^{k/2},
 \label{eq:teacher-feature-moments}\\
 \|Q(W)\|_{L^p}
   &\leq C(p-1)^k\sqrt{\log d},
 &\|g^\star(Q(W))\|_{L^p}
   &\leq C_g(p-1)^k\sqrt{\log d}.
 \label{eq:q-moments}
\end{align}
Here $C$ and all constants and logarithmic powers in the displayed
bounds may depend on $k,\varepsilon,L$ and the constants in
$d_1\asymp d^\varepsilon$, but not on $p$, $\bv$, or the particular
teacher realization in the event. The constant $C_g$ may additionally
depend on $|g^\star(0)|$ and $\operatorname{Lip}(g^\star)$.
Consequently, for every fixed $K>0$ there is a fixed $J>0$ such that,
with $R=(\log d)^J$, for all sufficiently large $d$,
\begin{equation}\label{eq:teacher-tails}
\sup_{\|\bv\|=1}\sum_{W\in\{u,Z\}}
 \mathbb P\!\left\{
 \begin{array}{l}
 |\bv^\top W|>R\ \text{or}\ \|W\|>R\sqrt{d_1}\ \text{or}\\
 |Q(W)|>R\ \text{or}\ |g^\star(Q(W))|>R
 \end{array}\right\}
 \leq d^{-K}.
\end{equation}
\end{lemma}

\begin{proof}
Apply Lemma~\ref{lem:1dclt}, including its Gram estimate
\eqref{eq:joint-clt-gram}, with $r=d_1$. Gaussian matrix concentration
gives the two bounds for $A^{(2)}$ in \eqref{eq:teacher-event}.
The failure probabilities can each be made smaller than any fixed
power of $d^{-1}$. We intersect these events with the event for the
label moments constructed next.

Fix $A^{(1)}$ with $\|G-I_{d_1}\|_2\leq1$, before conditioning on
$A^{(2)}$. Gaussian hypercontractivity gives
$\|u_i\|_{L^4_\bx}\leq3^{k/2}\sqrt2$.
By independence of the two layers, write
\[
 Q(u)=\frac{1}{\sqrt2d_1}\sum_i\xi_{ii}(u_i^2-1)
       +\frac{\sqrt2}{d_1}\sum_{i<j}\xi_{ij}u_i u_j,
 \qquad \xi_{ij}\stackrel{\mathrm{iid}}\sim\mathcal N(0,1).
\]
Thus $Q(u(\cdot))$ is a centered Gaussian element of $L^2(\gamma_d)$,
and
\[
 \mathbb E_{A^{(2)}}\|Q(u)\|_{L^2_\bx}^2
 =\frac{1}{2d_1^2}\sum_i\|u_i^2-1\|_{L^2_\bx}^2
  +\frac{2}{d_1^2}\sum_{i<j}\|u_i u_j\|_{L^2_\bx}^2
 \leq C_k.
\]
Minkowski's inequality and scalar Gaussian moments imply
\[
 \left(\mathbb E_{A^{(2)}}\|Q(u)\|_{L^2_\bx}^q\right)^{1/q}
 \leq C_k\sqrt q,\qquad q\geq2.
\]
Taking $q$ to be a sufficiently large multiple of $(L+1)\log d$
and applying Markov's inequality proves
$\|Q(u)\|_{L^2_\bx}\leq C\sqrt{\log d}$ except on a conditional
teacher event of probability at most $\tfrac13d^{-L}$.
Choosing the failure probabilities in the preceding coefficient
bounds to have the same margin proves the asserted common-event
probability.

Now fix both teachers in this event. Since $\|G\|_2\leq2$ for all
sufficiently large $d$, hypercontractivity gives
\[
 \|\bv^\top u\|_{L^p}
 \leq(p-1)^{k/2}(\bv^\top G\bv)^{1/2}.
\]
Minkowski's inequality gives
\[
 \big\|\|u\|\big\|_{L^p}
 \leq\left(\sum_i\|u_i\|_{L^p}^2\right)^{1/2}
 \leq(p-1)^{k/2}(\operatorname{tr}G)^{1/2}.
\]
The analogous Gaussian estimates for $Z$ prove
\eqref{eq:teacher-feature-moments} because $k\geq1$.
The polynomial $Q(u)$ has degree at most $2k$ in $\bx$, so
\[
 \|Q(u)\|_{L^p}
 \leq(p-1)^k\|Q(u)\|_{L^2}
 \leq C(p-1)^k\sqrt{\log d}.
\]
For $Z$, use $\|Q(Z)\|_{L^2}=\|A^{(2)}\|_F=\mathcal O(1)$ and
degree-two hypercontractivity. The growth bound
$|g^\star(t)|\leq|g^\star(0)|+\operatorname{Lip}(g^\star)|t|$
then proves \eqref{eq:q-moments}. Finally, take $p$ proportional to
$(K+1)\log d$ in these moment estimates, apply Markov's inequality,
and increase the fixed logarithmic exponent $J$ to obtain
\eqref{eq:teacher-tails}, uniformly in $\bv$.
\end{proof}

The Wasserstein bound in Lemma~\ref{lem:1dclt} applies to Lipschitz test functions, whereas the population estimators involve unbounded label-weighted quadratic moments. The next lemma uses the preceding moment bounds to compare these moments with their Gaussian counterparts, providing the population comparison used in both theorem proofs before applying Lemma~\ref{lem:comp_hermite}.

\begin{lemma}[Gaussian comparison for latent second moments]\label{lem:latent-moment-comparison}
Under the assumptions of Lemma~\ref{lem:teacher-moments}, let $u$, $Q$ and $Z$ be as defined there, and put
\[
 C^{\mathrm{lat}}=\mathbb E_\bx[g^\star(Q(u))H_2(u)],
 \qquad
 C_g^{\mathrm{lat}}=\mathbb E_Z[g^\star(Q(Z))H_2(Z)].
\]
These expectations are conditional on the teacher parameters. On the teacher event in Lemma~\ref{lem:teacher-moments},
\begin{equation}\label{eq:latent-moment-comparison}
 \|C^{\mathrm{lat}}-C_g^{\mathrm{lat}}\|_2
 =\tilde{\mathcal O}\left(\frac{d_1}{\sqrt d}\right).
\end{equation}
\end{lemma}

\begin{proof}
Fix the teacher parameters on the event in Lemma~\ref{lem:teacher-moments}, and suppress the subscripts on conditional expectations. Write $M=C^{\mathrm{lat}}-C_g^{\mathrm{lat}}$. Since $M$ is symmetric,
\[
 \|M\|_2=\sup_{\|\bv\|=1}|\bv^\top M\bv|.
\]
For every unit vector $\bv$, the quadratic form is the difference of expectations of
\[
 \Psi_\bv(w)=\mathrm{He}_2(\bv^\top w)g^\star(Q(w)),
 \qquad \mathrm{He}_2(t)=\frac{t^2-1}{\sqrt2},
\]
under $u$ and $Z$.

Fix $R=(\log d)^K$, with $K$ sufficiently large. Let $\Pi_R$ be the Euclidean projection onto the ball of radius $R\sqrt{d_1}$ and let $[t]_R=\max(-R,\min(t,R))$. Define
\[
 \Psi_{\bv,R}(w)
 =\mathrm{He}_2\!\left([\bv^\top\Pi_R w]_R\right)
   g^\star\!\left([Q(\Pi_R w)]_R\right).
\]
This function agrees with $\Psi_\bv$ on
\[
 \mathcal E_1(\bv)=
 \left\{w:|\bv^\top w|\leq R,\ |Q(w)|\leq R,
                   \ \|w\|\leq R\sqrt{d_1}\right\}.
\]
The maps $\Pi_R$ and $[\cdot]_R$ are $1$-Lipschitz. On the ball $\|w\|\leq R\sqrt{d_1}$, the operator-norm bound in \eqref{eq:teacher-event} gives
\[
 \|\nabla Q(w)\|=\|\sqrt2 A^{(2)}w\|
 \leq C\|A^{(2)}\|_2\|w\|=\mathcal O(R).
\]
The function $w\mapsto\mathrm{He}_2([\bv^\top\Pi_R w]_R)$ is bounded in absolute value by $\mathcal O(R^2)$ and has Lipschitz constant $\mathcal O(R)$, uniformly over unit $\bv$. Assumption~\ref{ass:2} therefore gives
\[
 \sup_{\|\bv\|=1}\operatorname{Lip}(\Psi_{\bv,R})
 =\tilde{\mathcal O}(1),\qquad
 \sup_{\|\bv\|=1}\|\Psi_{\bv,R}\|_\infty=\tilde{\mathcal O}(1).
\]

By \eqref{eq:teacher-tails}, for every fixed $L'>0$ and sufficiently large $K$,
\[
 \mathbb P\{u\notin\mathcal E_1(\bv)\}
 +\mathbb P\{Z\notin\mathcal E_1(\bv)\}\leq d^{-L'},
\]
uniformly over unit $\bv$. H\"older's inequality, the same moment bounds, and the bound on $\Psi_{\bv,R}$ give
\[
 \sup_{\|\bv\|=1}
 \left(\mathbb E|\Psi_\bv(u)-\Psi_{\bv,R}(u)|
       +\mathbb E|\Psi_\bv(Z)-\Psi_{\bv,R}(Z)|\right)
 \leq d^{-L'/2},
\]
after increasing $K$ if necessary. Consequently,
\[
 \|M\|_2
 \leq \sup_{\|\bv\|=1}\operatorname{Lip}(\Psi_{\bv,R})\,d_W(u,Z)
       +o(d^{-1})
 =\tilde{\mathcal O}\left(\frac{d_1}{\sqrt d}\right),
\]
where the last step uses the Wasserstein bound in Lemma~\ref{lem:teacher-moments}.
\end{proof}

\subsection{Proof of Theorem \ref{thm:matrix_conc_1}}

Condition on the teacher event in Lemma~\ref{lem:teacher-moments}; all data expectations and probabilities below are conditional on these parameters. We first bound the sample fluctuation. Write
\[
    \phi_\mu=F[H_k(\bx_\mu)]\in\mathbb{R}^{D_1},
    \qquad D_1=\mathcal{B}(d,k)=\mathcal{O}(d^k).
\]
Each summand of $\widehat C^{(1)}_k$ is $y_\mu(\phi_\mu\phi_\mu^\top-I_{D_1})/\sqrt2$. We apply Matrix Bernstein after truncating these summands, following the proof of Lemma F.5 in \cite{wen2025does}. Let
\[
    \mathcal{E}_\mu
    =
    \left\{\norm{\phi_\mu}^2\leq C_0D_1,\ |y_\mu|\leq C_y\right\},
\]
where $C_0,C_y$ may grow polylogarithmically with $d$. Lemma~\ref{lem:tail} and the label-moment bound \eqref{eq:q-moments} allow us to choose them so that $\mathbb P(\mathcal E_\mu^c)\leq d^{-K}$ for any fixed $K$. Define the truncated summand
\[
    Z_\mu
    =
    y_\mu H_2(\phi_\mu)\mathbf{1}_{\mathcal{E}_\mu},
    \qquad
    X_\mu
    =
    \frac{1}{n}\left(Z_\mu-\mathbb{E}Z_\mu\right).
\]
On $\mathcal{E}_\mu$, we have
\[
    \norm{Z_\mu}_2
    \leq
    |y_\mu|\left(\norm{\phi_\mu}^2+1\right)
    =
    \tilde{\mathcal{O}}(D_1),
\]
and hence the centered summands satisfy
\[
    \norm{X_\mu}_2
    \leq
    \tilde{\mathcal{O}}\left(\frac{D_1}{n}\right)
    =
    \tilde{\mathcal{O}}\left(\frac{d^k}{n}\right).
\]
Since $\mathbb{E}[\phi_\mu\phi_\mu^\top]=I_{D_1}$,
\[
\begin{aligned}
    \left\|\sum_{\mu=1}^n\mathbb{E}\,X_\mu^2\right\|_2
    &\leq
    \frac{1}{n}\left\|\mathbb{E}\,Z_\mu^2\right\|_2 \\
    &\leq
    \frac{1}{n}\,
    \tilde{\mathcal{O}}\left(
    D_1\left\|\mathbb{E}\,\phi_\mu\phi_\mu^\top\right\|_2+1
    \right)
    =
    \tilde{\mathcal{O}}\left(\frac{D_1}{n}\right)
    =
    \tilde{\mathcal{O}}\left(\frac{d^k}{n}\right).
\end{aligned}
\]
Matrix Bernstein applied to the independent centered matrices $\{X_\mu\}_{\mu=1}^n$ thus gives
\[
    \left\|
    \frac1n\sum_{\mu=1}^n Z_\mu-\mathbb{E}Z_\mu
    \right\|_2
    =
    \tilde{\mathcal{O}}\left(
    \sqrt{\frac{d^k}{n}}+\frac{d^k}{n}
    \right)
    =
    \tilde{\mathcal{O}}\left(\sqrt{\frac{d^k}{n}}\right),
\]
where the last step uses the regime $n\gtrsim d^k$. The union bound gives
$\mathbb P(\bigcup_{\mu=1}^n\mathcal E_\mu^c)\leq nd^{-K}$, while H\"older's inequality, Gaussian hypercontractivity for $\phi_\mu$, and the label moments in Lemma~\ref{lem:teacher-moments} imply
\[
 \|\mathbb E[y_\mu H_2(\phi_\mu)\mathbf1_{\mathcal E_\mu^c}]\|_2
 \leq\tilde{\mathcal O}(D_1)d^{-K/2}.
\]
Since $n$ grows polynomially, choosing $K$ sufficiently large makes both errors negligible, including after multiplication by $\sqrt{d_1}$. We obtain
\begin{equation}
    \norm{\widehat C^{(1)}_k-\mathbb{E}[\widehat C^{(1)}_k]}_2 = \tilde{\mathcal{O}}\left(\sqrt{\frac{d^{k}}{n}}\right)
\end{equation}

For the population term, use $u$, $G$ and $Q$ from Lemma~\ref{lem:teacher-moments} and the moments $C^{\mathrm{lat}},C_g^{\mathrm{lat}}$ from Lemma~\ref{lem:latent-moment-comparison}. Since $G=A^{(1)}A^{(1)\top}$,
\begin{equation}\label{eq:latent-centering}
 A^{(1)}\mathbb E[\widehat C_k^{(1)}]A^{(1)\top}
 =C^{\mathrm{lat}}+
   \frac{I_{d_1}-G}{\sqrt2}\mathbb E[g^\star(Q(u))].
\end{equation}
Lemma~\ref{lem:latent-moment-comparison} gives
\begin{equation}\label{eq:first-gaussian-comparison}
 \sqrt{d_1}\|C^{\mathrm{lat}}-C_g^{\mathrm{lat}}\|_2
 =\tilde{\mathcal O}\left(\frac{d_1^{3/2}}{\sqrt d}\right).
\end{equation}
By the Gram and label-moment bounds in Lemma~\ref{lem:teacher-moments}, the centering term in \eqref{eq:latent-centering} contributes at most
\[
 \frac{\sqrt{d_1}}{\sqrt2}\|I_{d_1}-G\|_2\,
               \mathbb E|g^\star(Q(u))|
 =\tilde{\mathcal O}\left(\frac{d_1}{d^{k/2}}\right),
\]
which is absorbed into \eqref{eq:first-gaussian-comparison} for $k\geq1$.
The bounds for $\|A^{(2)}\|_F$ and $\|A^{(2)}\|_2$ in Lemma~\ref{lem:teacher-moments} meet the hypotheses of Lemma~\ref{lem:comp_hermite} in dimension $d_1$. Thus
\[
 \sqrt{d_1}\|C_g^{\mathrm{lat}}-\nu_1A^{(2)}\|_2
 =\tilde{\mathcal O}(d_1^{-1/2}).
\]
Combining these bounds with the concentration estimate and $\|A^{(1)}\|_2^2=\|G\|_2=\mathcal O(1)$ from Lemma~\ref{lem:teacher-moments} proves the theorem.

\subsection{Proof of Theorem \ref{thm:matrix_conc_2}}

Condition on the teacher event in Lemma~\ref{lem:teacher-moments}; all expectations below are conditional on these parameters. We first bound the sample fluctuation
\[
    \tilde C^{(2)}_2-\mathbb E[\tilde C^{(2)}_2].
\]
The conditional moment bounds in that lemma give $\|h^{(1)}(\bx)\|\leq\tilde{\mathcal O}(\sqrt{d_1})$ and $|y|\leq\tilde{\mathcal O}(1)$ with conditional probability at least $1-d^{-K}$ for any fixed $K>0$. Truncating to these events, the summands $y_\mu H_2(h^{(1)}(\bx_\mu))$ have operator norm $\tilde{\mathcal O}(d_1)$, and their centered and normalized versions have norm $\tilde{\mathcal O}(d_1/n)$. Since
\[
 H_2(u)^2\preceq C(\|u\|^2uu^\top+I),
 \qquad \mathbb E[uu^\top]=G,\qquad \|G\|_2=\mathcal O(1),
\]
the matrix variance parameter is $\tilde{\mathcal O}(d_1/n)$. Matrix Bernstein applies to the conditionally independent centered summands. The moment bounds in Lemma~\ref{lem:teacher-moments}, H\"older's inequality and a union bound make removal of the cutoff negligible for polynomially growing $n$, giving
\begin{equation}\label{eq:b1}
 \|\tilde C^{(2)}_2-\mathbb E[\tilde C^{(2)}_2]\|_2
 =\tilde{\mathcal O}\left(\sqrt{\frac{d_1}{n}}+\frac{d_1}{n}\right)
 =\tilde{\mathcal O}\left(\sqrt{\frac{d_1}{n}}\right),
\end{equation}
where the last step uses $n\gtrsim d_1$.

Since $\mathbb E[\tilde C^{(2)}_2]=C^{\mathrm{lat}}$, Lemma~\ref{lem:latent-moment-comparison} gives
\begin{equation}\label{eq:b2}
 \sqrt{d_1}\|\mathbb E[\tilde C^{(2)}_2]-C_g^{\mathrm{lat}}\|_2
 =\tilde{\mathcal O}\left(\frac{d_1^{3/2}}{\sqrt d}\right).
\end{equation}
Lastly, the normalization and operator-norm bounds in Lemma~\ref{lem:teacher-moments} allow us to apply Lemma~\ref{lem:comp_hermite}, yielding
\[
 \sqrt{d_1}\|C_g^{\mathrm{lat}}-\nu_1A^{(2)}\|_2
 =\tilde{\mathcal O}(d_1^{-1/2}).
\]
Combining this with \eqref{eq:b1} and \eqref{eq:b2} completes the proof.

\subsection{Proof of Lemma \ref{lem:comp_hermite}}\label{app:comp-hermite-proof}

Let $z\sim\mathcal N(0,I_d)$ and write
\[
    p(z)=\langle A,H_2(z)\rangle,\qquad
    B=P_2(\mathcal{F}),
\]
so that the degree-$2$ projection of $\mathcal{F}$ is $\langle B,H_2(z)\rangle$. By orthonormality of the degree-$2$ Hermite features, for every symmetric matrix $M$,
\begin{equation}\label{eq:app:b-coeff}
    \langle B,M\rangle
    =
    \mathbb{E}\left[g^\star(p(z))\langle M,H_2(z)\rangle\right].
\end{equation}
It therefore suffices to prove, uniformly over all unit vectors $\bv$,
\begin{equation}\label{eq:app:directional-comp}
    \left|\langle B-\nu_1A,\bv\bv^\top\rangle\right|
    =\tilde{\mathcal{O}}\left(d^{-1}\right),
\end{equation}
because the supremum of the left-hand side is $\norm{B-\nu_1A}_2$.

Diagonalize $A=\sum_{i=1}^d a_iw_iw_i^\top$ and set $Z_i=w_i^\top z$. The variables $Z_i$ are independent standard Gaussians, and
\[
    p(z)=\sum_{i=1}^d a_i\frac{Z_i^2-1}{\sqrt{2}},\qquad
    \max_i|a_i|=\mathcal{O}(d^{-1/2}),\qquad
    \sum_{i=1}^d a_i^2=1+\tilde{\mathcal{O}}(d^{-1/2}).
\]
Sign symmetry in each coordinate gives $w_i^\top B w_j=0$ for $i\ne j$. Hence $B=\sum_i b_iw_iw_i^\top$, where
\[
    b_i=\mathbb E\left[g^\star(p(z))\frac{Z_i^2-1}{\sqrt{2}}\right].
\]

Let $L$ be the Lipschitz constant of $g^\star$. Its weak derivative $g^{\star\prime}$ exists almost everywhere and admits a measurable version bounded by $L$. For $a_i\ne0$, one Gaussian integration by parts in $Z_i$ gives
\begin{equation}\label{eq:app:weak-stein-comp}
    b_i=a_i\mathbb E\left[Z_i^2g^{\star\prime}(p(z))\right].
\end{equation}
Indeed, $\partial_{Z_i}(g^\star\circ p)=\sqrt{2}a_iZ_i g^{\star\prime}(p)$ almost everywhere; the Lipschitz bound ensures integrability and vanishing boundary terms. The law of $p(z)$ has a density, so the choice of the derivative on its exceptional null set is immaterial. If $a_i=0$, independence gives $b_i=0$ directly.

Multiplying the density of $Z_i^2\sim\chi_1^2$ by $Z_i^2$ gives the density of $\chi_3^2$. Thus, introducing two independent standard Gaussians $U,V$, independent of $z$, and setting
\[
    p_i^{\mathrm{sb}}=p(z)+\frac{a_i}{\sqrt{2}}(U^2+V^2),
\]
we can rewrite \eqref{eq:app:weak-stein-comp} as
\begin{equation}\label{eq:app:size-biased-comp}
    b_i=a_i\mathbb E\left[g^{\star\prime}(p_i^{\mathrm{sb}})\right].
\end{equation}
The mean and variance of $p_i^{\mathrm{sb}}$ are
\[
    m_i=\sqrt{2}a_i,\qquad
    \sigma_i^2=\sum_{j=1}^d a_j^2+2a_i^2.
\]
The centered, standardized variable $F_i=(p_i^{\mathrm{sb}}-m_i)/\sigma_i$ belongs to the second Wiener chaos. The normal approximation bound in total variation from Theorem~3.1 and Proposition~3.2 of \cite{nourdin2009stein} therefore gives, for $G\sim\mathcal N(0,1)$,
\begin{equation}\label{eq:app:tv-comp}
\begin{split}
    d_{\mathrm{TV}}(F_i,G)
    &\leq\sqrt{\frac{2}{3}}\sqrt{\mathbb E[F_i^4]-3}\\
    &=\frac{2\sqrt{2}}{\sigma_i^2}
      \left(\sum_{j=1}^d a_j^4+2a_i^4\right)^{1/2}
      =\mathcal{O}(d^{-1/2}).
\end{split}
\end{equation}
Here $d_{\mathrm{TV}}$ is the supremum of the difference of probabilities over Borel sets. The displayed fourth moment follows by independence and the identity $\mathbb E[((G^2-1)/\sqrt{2})^4]-3=12$. The bound is uniform in $i$, since $\sum_j a_j^4\leq\norm{A}_2^2\norm{A}_F^2=\mathcal O(d^{-1})$ and $\sigma_i^2$ is bounded away from zero for all sufficiently large $d$.

Moreover, $|m_i|=\mathcal O(d^{-1/2})$ and $|\sigma_i^2-1|=\tilde{\mathcal O}(d^{-1/2})$, uniformly in $i$. Comparing Gaussian densities and using \eqref{eq:app:tv-comp} yields
\[
    d_{\mathrm{TV}}(p_i^{\mathrm{sb}},G)
    \leq d_{\mathrm{TV}}(F_i,G)
       +d_{\mathrm{TV}}(m_i+\sigma_iG,G)
    =\tilde{\mathcal O}(d^{-1/2}).
\]
Finally, the one-dimensional Gaussian integration-by-parts identity for Lipschitz functions gives $\nu_1=\mathbb E[g^\star(G)G]=\mathbb E[g^{\star\prime}(G)]$. Since $\norm{g^{\star\prime}}_\infty\leq L$, \eqref{eq:app:size-biased-comp} implies
\[
    |b_i-\nu_1a_i|
    \leq 2L|a_i|\,d_{\mathrm{TV}}(p_i^{\mathrm{sb}},G)
    =\tilde{\mathcal O}(d^{-1}),
\]
uniformly over $i$. Consequently,
\[
    \norm{P_2(\mathcal F)-\nu_1A}_2
    =\max_i|b_i-\nu_1a_i|
    =\tilde{\mathcal O}(d^{-1}).
\]
Multiplying by $\sqrt d$ gives the claimed bound.

\section{Additional numerical experiments}
\label{sec:app:numerics}
We describe the evaluation protocol and report additional numerical experiments.

\paragraph{Eigenvector overlap.}
After training Algorithm~\ref{alg:agnostic-recovery}, let $\{\widehat A_i^{(1)}\}_{i=1}^{\widehat d_1}$ and $\widehat A^{(2)}$ denote the estimates of $\{A_i^{(1)}\}_{i=1}^{d_1}$ and $A^{(2)}$, respectively.

Under the subspace-recovery conjecture in Remark~\ref{rem:tensor}, the asymptotic orthonormality of the teacher rows motivates the idealized rotation relation, for orthogonal $R\in\mathbb R^{d_1\times d_1}$,
\begin{equation}
    \widehat A^{(1)} = R\,A^{(1)}.
\end{equation}
We measure similarity using the Frobenius overlap
\begin{equation}
    q^{(1)}_{A} = \| A^{(1)} (\widehat A^{(1)})^{\top} \|_F 
\end{equation}

Under this idealized rotation relation, $\widehat{\bh}^{(1)} = R\,\bh^{(1)}$. Plugging this expression into $\hat{C}^{(2)}_2$ in Eq.~\eqref{eq:main:second_stage_moment} and using $H_2(R\bx)=R H_2(\bx)R^\top$ gives
\begin{align}
\label{eq:app:rotation_estimation}
    \widehat C^{(2)}_2
    = \frac{1}{n}\sum_{\mu=1}^n y_\mu H_2(\widehat \bh^{(1)}_\mu)
    = R\left(\frac{1}{n}\sum_{\mu=1}^n y_\mu H_2(\bh^{(1)}_\mu)\right)R^\top ,
\end{align}
Theorem~\ref{thm:matrix_conc_2} therefore predicts alignment with $\nu_1 R A^{(2)}R^\top$ in this idealized setting.

\begin{figure}
   \centering
\includegraphics[width=0.7\linewidth]{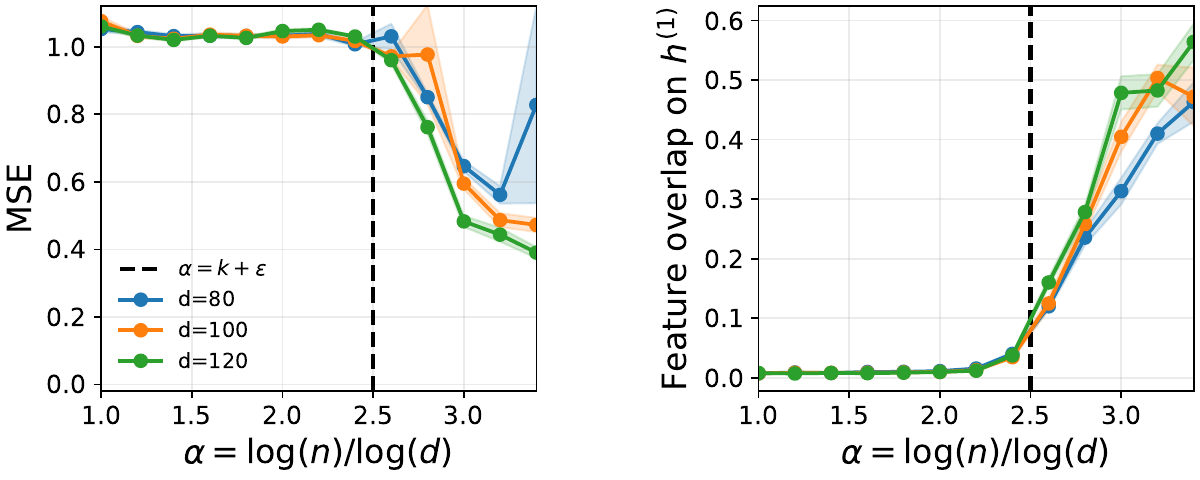}
   \caption{\textbf{Learning with a centered ReLU non-linearity.}
    Performance of the hierarchical estimator described in Algorithm~\ref{alg:agnostic-recovery} when learning the target~\eqref{eq:main:simple_target} with $g^\star$ chosen as a centered ReLU function.
    \textbf{Left:} MSE versus the sample-size exponent $\alpha=\log_d n$, for $d\in\{80,100,120\}$ and $d_1=\sqrt d$. The MSE drops near the conjectured scale $\alpha=k+\varepsilon$. \textbf{Right:} First-layer feature overlap, which increases near the same threshold.}
   \label{fig:gcrelu}
\end{figure}

\paragraph{Latent feature overlap.}
On a test dataset $\{(\bx_\mu,y_\mu)\}_{\mu=1}^{n_{\mathrm{test}}}$, form the estimated feature matrix $\hat H^{(1)}\in\mathbb R^{n_{\mathrm{test}}\times\widehat d_1}$ with rows $\widehat{\bh}^{(1)\top}_\mu$. We compare it with the ground-truth feature matrix $H^{(1)}\in\mathbb R^{n_{\mathrm{test}}\times d_1}$ using
\begin{equation}
    q^{(1)}_{\bh} = \| \bh^{(1)} (\widehat{\bh}^{(1)})^\top \|_F^2 .
\end{equation}
The idealized relations $\widehat{\bh}^{(1)}=R\bh^{(1)}$ and $\widehat A^{(2)}=\nu_1 R A^{(2)}R^\top$ also give
\begin{align}
    \widehat h_\mu^{(2)}
    &=\langle \widehat A^{(2)},H_2(\widehat\bh^{(1)}_\mu)\rangle_F\\
    &=\nu_1\langle R A^{(2)}R^\top,R H_2(\bh^{(1)}_\mu)R^\top\rangle_F
      =\nu_1 h_\mu^{(2)}.
\end{align}
The final regression can absorb the scale by learning $t\mapsto g^\star(t/\nu_1)$, since $\nu_1\ne0$ by Assumption~\ref{ass:2}.
Feature-overlap curves average over $10$ seeds; error bars show one standard deviation.
\paragraph{Computing the MSE.}
Given the estimated second-layer features, we fit the labels by kernel regression on $\mathcal D_2=\{(\widehat h^{(2)}_\mu,y_\mu)\}_{\mu=1}^n$, using the penalized empirical risk
\begin{align}
    \mathcal{R}(\mathbf{a};\mathcal{D}_2) = \sum_{\mu=1}^n \left( y_\mu - \langle \mathbf{a},\varphi(\widehat{h}^{(2)}_\mu)\rangle \right)^2 + \lambda ||\mathbf{a}||_2
\end{align}
Here $\mathbf a\in\mathbb R^p$ denotes the predictor coefficients for the feature map $\varphi:\mathbb R\to\mathbb R^p$. Minimizing this risk gives $\hat y(h)=\langle\hat{\mathbf a},\varphi(h)\rangle$. We use a degree-$7$ polynomial kernel and select the regularization parameter from $\{10^{-3},10^{-4},10^{-5}\}$ by cross-validation.

We evaluate prediction error on a fresh test dataset $\mathcal D_{\mathrm{test}}$:
\begin{align}
    \mathrm{MSE }= \frac{1}{n_{\mathrm{test}}}\sum_{\mu=1}^{n_{\rm test}} (\hat y_\mu - y_\mu)^2
\end{align}
We normalize the MSE by the label variance.
\paragraph{Varying the non-linearity.}
Figure~\ref{fig:gcrelu} extends the identity and $\tanh$ experiments in Figs.~\ref{fig:main_fig_subplots} and~\ref{fig:eps05_gtanh} to the centered ReLU function, $g^\star(x)=\mathrm{relu}(x)-\mathbb E_{\xi\sim\mathcal N(0,1)}[\mathrm{relu}(\xi)]$.

\paragraph{Higher Hermite order in the first layer.}
The main-text experiments use $k=2$. Figure~\ref{fig:k3} reports results for $k=3$, for which the $\mathcal O(d^k)$ feature dimension limits the input dimensions accessible numerically.

For $k=3$, the symmetric tensor space has dimension $D_1=\mathcal B(d,3)=d(d+1)(d+2)/6$. This scales as $d^3$, but its leading coefficient affects the transition at the simulated dimensions. We therefore use the finite-size prediction
\begin{equation}
    \alpha_{\rm th}(d) = \log_d\!\left( \frac{d(d+1)(d+2)}{6} \right) + \varepsilon ,
\end{equation}
which converges to $3+\varepsilon$ as $d\to\infty$. Figure~\ref{fig:k3} compares this prediction with the numerical results.

\begin{figure}
   \centering
\includegraphics[width=0.7\linewidth]{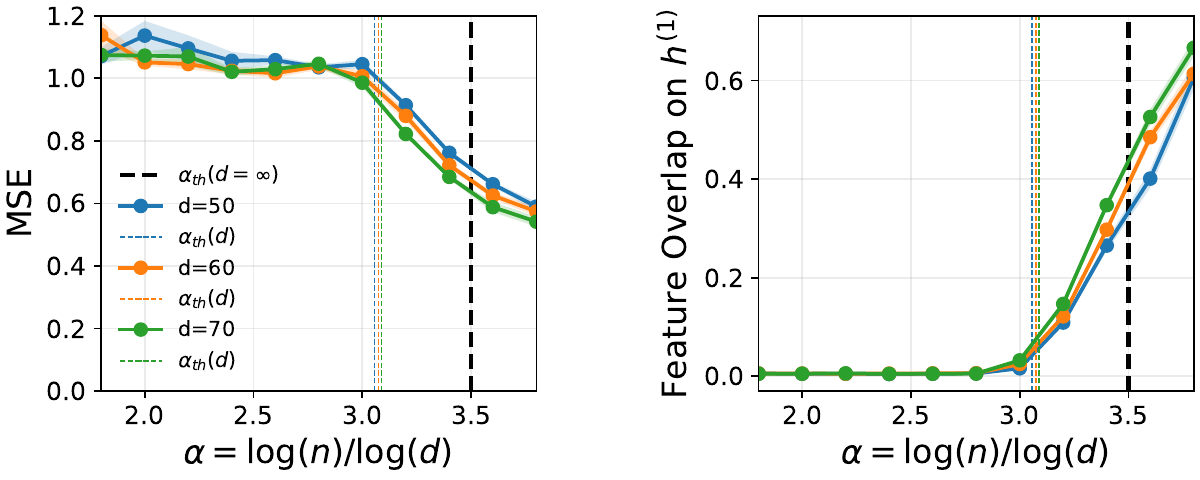}
   \caption{\textbf{Learning with a third-order Hermite polynomial ($k=3$).}
    Performance of the hierarchical estimator described in Algorithm~\ref{alg:agnostic-recovery} when learning the target~\eqref{eq:main:simple_target} with a third-order Hermite polynomial in the first layer.
    \textbf{Left:} MSE versus the sample-size exponent $\alpha=\log_d n$, for $d\in\{50,60,70\}$ and $d_1=\sqrt d$. Colored dashed lines indicate the finite-size predictions $\alpha_{\mathrm{th}}(d)$; the black dashed line marks the conjectured asymptotic threshold $\alpha=k+\varepsilon$.
    \textbf{Right:} Feature overlap between the learned first-layer representations $\{\widehat h^{(1)}_\mu\}_{\mu=1}^n$ and the ground truth (see Appendix~\ref{sec:app:numerics}). Both the MSE and the feature overlap exhibit a clear transition around the finite-size thresholds, with stronger finite-size effects compared to the $k=2$ case.
    }
   \label{fig:k3}
\end{figure}

\paragraph{Backward feature correction.} For the two-fold quadratic target in Eq.~\eqref{eq:main:simple_target}, the candidate signal \eqref{eq:mean_C1} depends on $A^{(2)}$, so later-layer estimates may inform a refinement of the first-layer representation, as in \cite{allen2023backward}. Whether such a refinement identifies the original teacher basis remains open.

\begingroup
\setlength{\abovedisplayskip}{5pt plus 1pt minus 1pt}
\setlength{\belowdisplayskip}{5pt plus 1pt minus 1pt}
\setlength{\belowdisplayshortskip}{3pt plus 1pt minus 1pt}
\section{Spectral learning as implicit gradient descent}
\label{sec:gd-spectral}
We examine how the empirical moment matrices used by our spectral estimator enter the early-time gradient dynamics of a matched architecture. Related connections for the last layer underlie the spectral methods in \cite{lu2020phase,maillard2022construction,mondelli2018fundamental,defilippis2024dimension,kovavcevic2025spectral} and have also been used with transformed targets in \cite{bonnaire2025role,zhang2025neuralnetworks}.

\subsection{Gradient-based setting}

We consider a matched student-teacher architecture at small random initialization and compute the leading terms in its gradient updates.

Flatten the first-layer Hermite features as $\boldsymbol{\phi}_\mu=F[H_k(\bx_\mu)]\in\R^{d^k}$ and write $F[\{A_i\}^{d^\varepsilon}_{i=1}]=A^{(1)}\in\R^{d^\varepsilon\times d^k}$. The teacher is then
\begin{align}
    y_\mu = \langle A^{(2)} , (A^{(1)} \boldsymbol{\phi}_\mu) (A^{(1)} \boldsymbol{\phi}_\mu)^T  - I_p \rangle \, .
\end{align}
At time $t\geq0$, let $\widehat A^{(1)}_t\in\R^{d^\varepsilon\times d^k}$ and $\widehat A^{(2)}_t$ denote the student parameters, with prediction
\begin{align}
    \widehat y_t \equiv \widehat y_t(\widehat A^{(1)}_t, \widehat A^{(2)}_t) := \langle \widehat A^{(2)} , (\widehat A^{(1)} \boldsymbol{\phi}_\mu) (\widehat A^{(1)} \boldsymbol{\phi}_\mu)^T  - I_p \rangle \, .
\end{align}
We use the squared loss
\begin{align}
    \mathcal{L}\equiv\mathcal{L}(A^{(1)}, A^{(2)}) = \frac{1}{2n}\sum_\mu \left(y_\mu - \widehat y_\mu(A^{(1)}, A^{(2)})\right)^2 \, .
\end{align}
Gradient descent updates the student parameters as
\begin{align}
    \forall i \in [p],\;  \widehat A^{(1)}_{i, t+1} = \widehat A^{(1)}_{i, t} - \eta \boldsymbol \nabla_{A^{(1)}_{i}} \mathcal{L}(\widehat A^{(1)}_t, \widehat A^{(2)}_t) \, , \\
    \forall i \in [p],\;  \widehat A^{(2)}_{i, t+1} = \widehat A^{(2)}_{i, t} - \eta \boldsymbol \nabla_{A^{(2)}_{i}} \mathcal{L}(\widehat A^{(1)}_t, \widehat A^{(2)}_t) \, .
\end{align}

\subsection{Emergence of the first-layer estimator}

Fix a nonzero second-layer estimate $\widehat A^{(2)}_t$. The first-layer gradient is
\begin{align}
    \boldsymbol \nabla_{A^{(1)}_{i}} \mathcal{L} &= \frac{-2}{n} \sum_\mu \left(y_\mu - \widehat y_\mu\right) (\widehat {\mathbf{A}}_{i,t}^{(2)})^T \widehat A^{(1)}_t \left( H_2(\boldsymbol{\phi}_\mu) + I_{d^k}\right) \\
    &= -2 (\widehat {\mathbf{A}}_{i,t}^{(2)})^T \widehat A^{(1)}_t \left( \widehat C^{(1)}_k + I_{d^k}\right) + (\widehat {\mathbf{A}}_{i,t}^{(2)})^T \widehat A_t^{(1)} \left( \frac{2}{n} \sum_\mu \widehat y_\mu\left( H_2(\boldsymbol{\phi}_\mu) + I_{d^k}\right)\right)\, .
\end{align}
The leading term involves the empirical moment matrix $\widehat C^{(1)}_k$ introduced in the main text. Differentiating again gives
\begin{align}
    \mathcal{H}_{ij} =  \boldsymbol \nabla_{A^{(1)}_{j}}\boldsymbol \nabla_{A^{(1)}_{i}}^T \mathcal{L} &= - 2\widehat A^{(2)}_{ij} (\widehat C^{(1)}_k + I_{d^k}) + \widehat A^{(2)}_{ij} \left( \frac{2}{n} \sum_\mu \widehat y_\mu\left( H_2(\boldsymbol{\phi}_\mu) + I_{d^k}\right)\right) \\
    &+ (\widehat {\mathbf{A}}_{i,t}^{(2)})^T \widehat A_t^{(1)} \left( \frac{2}{n} \sum_\mu \left( H_2(\boldsymbol{\phi}_\mu) + I_{d^k}\right)^2 \right) (\widehat A_t^{(1)})^T \widehat {\mathbf{A}}_{j,t}^{(2)} \, .
\end{align}
At initialization ($\widehat A^{(1)}_{t=0}\simeq0$), this becomes
\begin{align}
    \mathcal{H}_{t=0} \simeq \widehat A^{(2)}_{t=0} \left( \widehat C^{(1)}_k + I_{d^k} \right) \, .
\end{align}
This linear dependence on $\widehat C^{(1)}_k$ motivates the comparison with power iteration.

\subsection{Gradient structure for the second layer}

The second-layer gradient is
\begin{align}
    \nabla_{A^{(2)}} \mathcal{L}_t = \frac{-1}{n}\sum_\mu (y_\mu - \widehat y_\mu) H_2(\widehat \bh_\mu^{(1)}) = - \; \widehat C^{(2)}_2 + \frac{1}{n}\sum_\mu \widehat y_\mu H_2(\widehat \bh_\mu^{(1)}) \, .
\end{align}
At initialization, with $\widehat A^{(2)}_{t=0}\simeq0$,
\begin{align}
    \nabla_{A^{(2)}} \mathcal{L}_{t=0} = - \, \widehat C^{(2)}_2 \, .
\end{align}
Thus the empirical matrix $\widehat C^{(2)}_2$ drives the initial second-layer update.

\endgroup
\end{document}